\documentclass[preprint,11pt,authoryear,nonatbib]{elsarticle}
\usepackage[a4paper,margin=1in]{geometry}
\usepackage{setspace}
\usepackage{amsmath,amssymb,amsthm,bm}
\usepackage{graphicx,multirow,booktabs,makecell}
\usepackage{algorithm,algorithmicx,algpseudocode}
\providecommand{\IfDocumentMetadataT}[1]{}

\makeatletter\let\c@author\relax\makeatother
\usepackage[backend=biber,style=apa,natbib=true]{biblatex}
\usepackage[hidelinks]{hyperref}
\newtheorem{proposition}{Proposition}
\newtheorem{theorem}{Theorem}
\newtheorem{corollary}{Corollary}
\newtheorem{lemma}{Lemma}[section]

\makeatletter
\def\ps@pprintTitle{\let\@oddhead\@empty\let\@evenhead\@empty
  \def\@oddfoot{\hfil\thepage\hfil}\let\@evenfoot\@oddfoot}
\makeatother
\begin{document}
\begin{frontmatter}
\title{Learning to Design Convex Relaxations for Nonconvex Optimization}
\author[a]{Jinrun Liu}
\author[b]{Bai Yan}
\author[a]{Tengfei Liu}
\author[b]{Qunfeng Liu}
\author[a]{Jian Yang}
\author[c]{Shi Cheng}
\author[a]{Yuhui Shi}
\author[a]{Qi Zhao\corref{cor}}
\ead{zhaoq@sustech.edu.cn}
\cortext[cor]{Corresponding author.}
\address[a]{Department of Computer Science and Engineering, Southern University of Science and Technology, Shenzhen 518055, China}
\address[b]{School of Computer Science and Technology, Dongguan University of Technology, Dongguan 523808, China}
\address[c]{School of Artificial Intelligence and Computer Science, Shaanxi Normal University, Xi'an 710119, China}
\begin{abstract}
Convex relaxations make nonconvex optimization tractable by replacing hard nonconvex problems with convex surrogates that give valid bounds. Yet their design is still largely manual: a human expert must decide where to relax, which relaxation to use, and how much formulation growth is worth paying. We show that this design process can be learned. The method proposed in this paper, AutoConvexRelax, represents the symbolic formulation of a nonconvex problem as a tripartite graph linking variables, symbolic terms, and their objective or constraint positions. It then learns a policy that composes symbolic rewrites into problem-specific relaxation trajectories as the formulation evolves. We prove that finite action sequences generated by the learned policy preserve relaxation validity under explicit domain, occurrence-context, and representation-consistency conditions. We further establish computable error bounds for mixed relaxations and sufficient conditions for feasible recovery and global-optimality certification. As a first instantiation, we apply the method to nonconvex quadratic problems and fractional problems. On 241 benchmark quadratic instances, the learned policy strengthens McCormick lower bounds by 13.52\%. Its mean lower bound is 3.70\% below strengthened semidefinite relaxations on the normalized comparison metric, while it uses substantially smaller formulations. The learned policy retains this trade-off on fractional instances and reaches near-semidefinite bounds on beamforming and sensor-network localization instances. These results suggest that convex relaxation design can move from a fixed handcrafted choice to an automated learning process that adapts to problem structures. Code is available at \url{https://github.com/auto4opt/AutoConvexRelax}.

\end{abstract}
\end{frontmatter}
\setstretch{1.5}

\section{Introduction}\label{sec_introduction}

Nonconvex optimization appears throughout science and engineering, from signal processing and communications to control, machine learning, and scientific computing \citep{floudas2000deterministic,belotti2013mixed}. Its difficulty comes from nonconvex couplings between variables, such as bilinear products, indefinite quadratics, polynomial expressions, or more general nonlinear compositions \citep{tawarmalani2002convex}. Convex relaxation handles the difficulty by replacing the nonconvex structure with a tractable convex surrogate that gives a valid bound \citep{mccormick1976computability}. A set of relaxation tools has been developed, such as McCormick envelopes, moment relaxations, semidefinite relaxations, and branch-and-bound methods based on relaxation \citep{lasserre2001global,luo2010semidefinite}. This makes relaxation design nontrivial. One must choose where to relax, how to relax, and how much formulation growth to allow.

These choices remain largely manual. While solvers can automate standard envelopes, cuts, and node relaxations once a formulation is fixed, the choices considered in this paper arise during formulation construction: which terms to relax, how to relax them, and how much formulation growth to allow. They remain specified by human experts. No single relaxation strategy is best for all nonconvex terms. The right choice depends on the algebraic form of a term, the variables it couples, its role in the objective or constraints, and the cost of the added variables and constraints. Such dependencies are hard to encode as reusable rules. In practice, one term may be relaxed by McCormick envelopes, another by semidefinite lifting \citep{parrilo2003semidefinite,goemans1995improved}, another by quadratic convex reformulation (QCR) \citep{billionnet2009improving}, and another left unchanged. The rationale behind these choices often remains in expert practice rather than in a reusable procedure. This makes relaxation design slow, hard to reproduce, and difficult to transfer.

This motivates automated relaxation design. The goal is not to select one relaxation family for an entire problem, such as McCormick envelopes, semidefinite relaxation, or QCR. Nor is it to tune parameters within a fixed scheme. The goal is to construct a relaxation trajectory. At each step, the procedure chooses a target nonconvex term, applies a valid relaxation action, and updates the formulation. This keeps known relaxation tools, but shifts their placement and ordering from expert specification to learned decisions. The resulting trajectories can move beyond fixed single-tool schemes and improve the bound--cost trade-off. 

Learning has improved solver decisions and relaxation-oriented choices, including branching, cut selection, relaxation-family selection for quadratically constrained quadratic programs (QCQPs), partitioning, and linearization \citep{bengio2021machine,dedeoglu2026learning,kannan2025strong,baltean2018scoring,bonami2022classifier}. A separate line uses large language models (LLMs) to generate relaxations during inference \citep{liu2026autoco,peng2026nc2c,kim2026relaxations}. These developments motivate a further question: can relaxation design be learned as a validity-preserving and reusable construction process over evolving symbolic formulations? Section~\ref{sec_related_work} analyzes how this task differs from decisions within prescribed relaxation schemes and from generate--then--verify search.

Here we make such relaxation trajectories learnable. Starting from the symbolic formulation of a nonconvex problem, the learned policy repeatedly chooses a target term and an admissible relaxation action, rewrites the formulation, and stops when a convex surrogate is obtained. We refer to this method as AutoConvexRelax. The key is to give the policy a decision state that preserves the symbolic semantics needed for relaxation design. A coefficient-level view is not enough; the policy must know the type of each term, the variables it couples, the domains of these variables, and whether the term appears in the objective or in a constraint. We therefore represent each formulation state as a tripartite graph with nodes for variables, symbolic terms, and objective or constraint positions. This graph keeps symbolic terms as explicit decision targets and links them to their variables and formulation positions. A graph neural encoder turns this representation into the state used by the policy.

The second ingredient is an executable action system for relaxation. We do not ask the policy model to invent algebraic rewrites. Instead, we implement known relaxation principles as symbolic transformations with explicit applicability rules. Each action rewrites the formulation only when its validity conditions are satisfied, and then updates the admissible action set. This turns relaxation theory into a composable action space for learning. The action system enforces bound validity, while the policy learns where to apply admissible actions and how to order them.
The policy is trained through solver feedback. The reward favors tighter lower bounds and penalizes formulation growth from auxiliary variables and constraints.

We establish theoretical guarantees for the composition of heterogeneous relaxation actions. Under explicit domain, occurrence-context, and shared-representation conditions, we prove that any finite action sequence preserves valid bound semantics, independently of policy quality. We further derive computable objective and constraint discrepancy bounds by propagating local errors through mixed quadratic and fractional representations. Together with original-domain checks and a certified relaxation lower bound, these bounds yield sufficient conditions for feasible recovery and an optimality-gap certificate.

Across 241 held-out QCQPs, 15 held-out fractional instances, and application-derived beamforming and sensor network localization problems, the learned policy constructs mixed trajectories that apply McCormick relaxations, semidefinite lifting, QCR, and integrality relaxation to different parts of a formulation. These trajectories tighten lightweight relaxations and approach semidefinite bounds without committing to a full semidefinite construction.

The proposed AutoConvexRelax provides a general learning framework for convex relaxation design in nonconvex optimization. By separating the learned construction process from the operator library, the framework is not tied to the QCQP and fractional instantiations studied here. Extending the library with suitable validity-checked operators allows the same construction and learning mechanism to address broader classes of nonconvex problems.

\section{Related Work}\label{sec_related_work}
\subsection{Classical relaxations and formulation design}
McCormick envelopes, semidefinite lifting, moment relaxations, and reformulation--linearization techniques provide established constructions for valid bounds \citep{mccormick1976computability,luo2010semidefinite,lasserre2001global,parrilo2003semidefinite,sherali1990hierarchy,anstreicher2008semidefinite}. Quadratic convex reformulation exploits quadratic structure through convex under-estimation \citep{billionnet2009improving}. These mechanisms offer different strength--size trade-offs. General-purpose solvers automate envelopes, cuts, and node relaxations \citep{sahinidis1996baron,belotti2013mixed}, while formulation construction still requires choices about operator placement, combination, and ordering. Such choices depend on term structure, variable domains, and objective or constraint context.

\subsection{Learning within optimization and relaxation schemes}
Learning-based optimization guides branching \citep{khalil2016learning,gasse2019exact}, cut selection \citep{paulus2022learning}, and solution heuristics \citep{nair2021solving,xiong2024neuralqp}, improving decisions within a supplied formulation \citep{bengio2021machine}. These methods learn how to solve a given model, with the formulation serving as their starting point.

Several methods directly learn relaxation decisions. \citet{dedeoglu2026learning} select between LP and SDP relaxations for an entire QCQP, adapting the family choice while retaining predefined candidate constructions. \citet{kannan2025strong} learn partitioning for piecewise McCormick relaxations, refining envelopes within that scheme. \citet{baltean2018scoring} score sparse cuts for semidefinite outer approximation, adapting cut selection within a prescribed construction. \citet{bonami2022classifier} decide when to linearize mixed-integer quadratic formulations inside a solver. Their learned decisions concern family selection, partition refinement, cut selection, or linearization within specified relaxation schemes and solver workflows.

\subsection{LLM-generated relaxations}
LLM-based methods generate relaxations during inference. AutoCO \citep{liu2026autoco} relaxes constraints for heuristic search, serving a different purpose from bounding the original optimum. NC2C \citep{peng2026nc2c} generates convexified models without establishing their optima as valid bounds. The dual-agent method \citep{kim2026relaxations} evaluates generated tightening constraints through another LLM and human review, then certifies the numerical bound by duality. Generated models and constraints therefore require output-specific verification; evaluations are consumed within each inference-time search.

\subsection{Analysis}\label{sec_related_analysis}
The methods reviewed above automate aspects of optimization, but leave relaxation construction partly prescribed or dependent on output-specific verification. Classical approaches provide relaxation mechanisms whose placement and combination require formulation design. Learning-based approaches improve the solution process \citep{khalil2016learning,gasse2019exact,nair2021solving,bengio2021machine,paulus2022learning,xiong2024neuralqp} or select and refine relaxations within prescribed schemes \citep{baltean2018scoring,bonami2022classifier,dedeoglu2026learning,kannan2025strong}. These objectives do not encompass constructing problem-specific relaxations by combining different mechanisms across heterogeneous terms. LLM-based approaches generate models or constraints at inference time, but target heuristic search \citep{liu2026autoco}, convexification \citep{peng2026nc2c}, or generation followed by verification \citep{kim2026relaxations}. The remaining question is whether learning can go beyond isolated decisions within prescribed schemes and generate--then--verify search to make relaxation design a validity-preserving and reusable construction process.

AutoConvexRelax addresses this question by making problem-specific relaxation construction itself the learning objective. It moves beyond selecting a relaxation family or tuning a fixed scheme to learning how established relaxation mechanisms should be combined and organized for a given formulation. This enables automated relaxation design that pursues improved bound--cost trade-offs while preserving valid bound semantics.

\section{The Proposed AutoConvexRelax Method}\label{sec_method}
The AutoConvexRelax method proposed in this paper treats convex relaxation design as a sequence of validity-preserving symbolic transformations. A state is the current symbolic formulation derived from the original nonconvex problem: initially the original formulation, and later a formulation modified by previous relaxation actions. An action specifies which term to relax and which admissible operator to apply. Applying an action rewrites the formulation and defines the next state. A successful trajectory ends with a convex surrogate, which is then solved to evaluate its bound quality and formulation cost.

Fig.~\ref{fig:overview} summarizes the pipeline. First, the symbolic problem formulation is encoded as a tripartite graph with nodes for variables, terms, and objective or constraint positions. This graph keeps each nonconvex term as an explicit decision target and records its modeling context. It lets the policy distinguish terms with similar coefficients but different variables, domains, algebraic roles, or constraint positions.
\begin{figure}[htbp]
\centering
\includegraphics[width=1.0\linewidth]{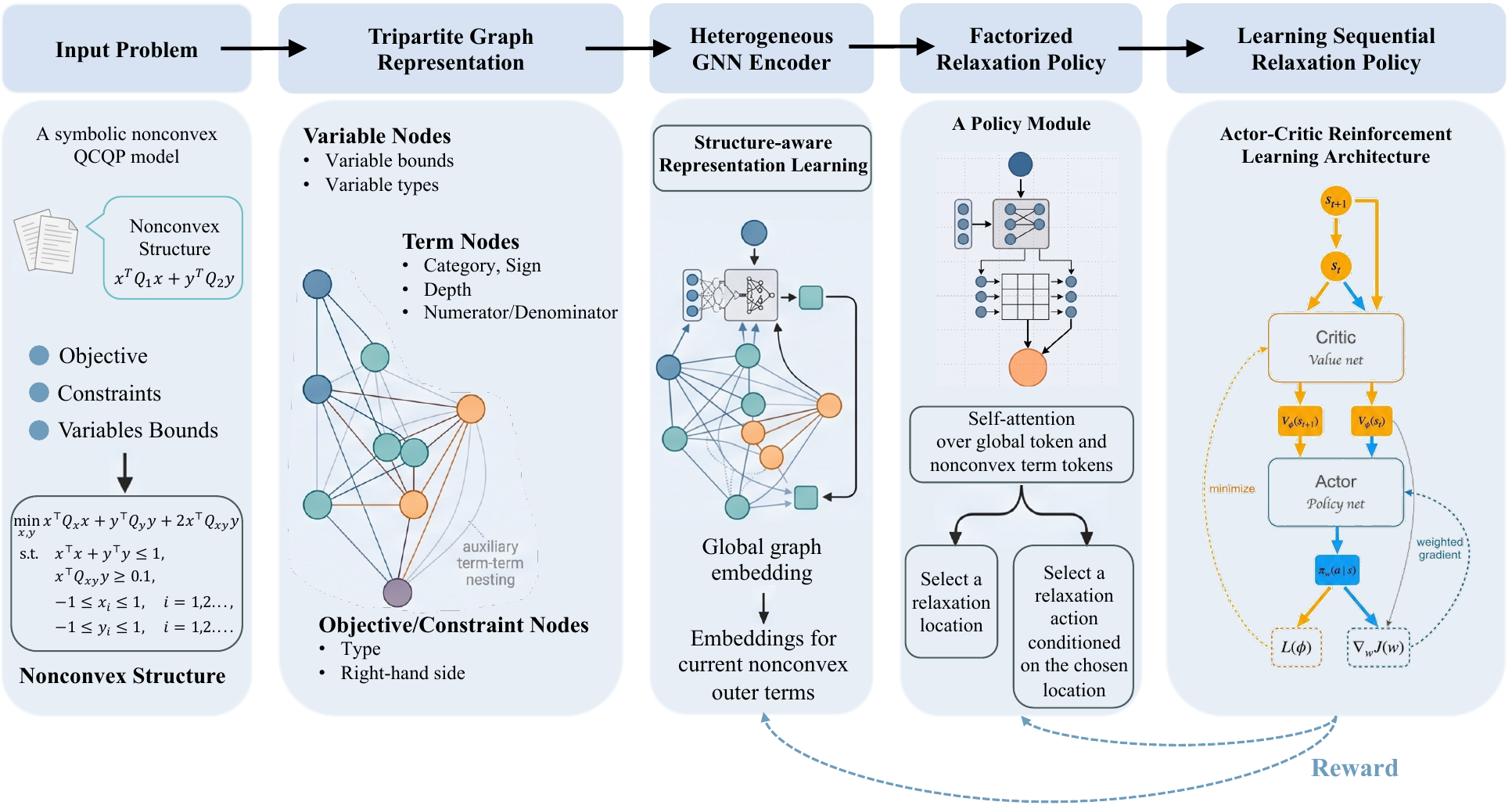}
\caption{Overview of the AutoConvexRelax method proposed in this paper. A symbolic nonconvex problem is encoded as a tripartite graph. A factorized policy selects a relaxation location and an admissible relaxation operator. The resulting valid transformations construct a convex surrogate whose bound quality and formulation cost are evaluated by a solver.}
\label{fig:overview}
\end{figure}

Second, the policy uses this representation to select a nonconvex target term and then chooses a relaxation operator conditioned on that target. This factorization matches the design task: a useful relaxation depends not only on the operator itself, but also on where it is applied and what previous rewrites have changed; the same operator may be useful for one term and unnecessary for another.

Third, validity is enforced by the action system. The policy does not invent relaxation constraints. It chooses only from admissible actions whose applicability conditions are satisfied at the current formulation state. Each accepted rewrite is therefore a valid relaxation step. Learning chooses and orders valid transformations; relaxation validity is preserved by construction.

\subsection{Problem setting and framework applicability}\label{sec_problem_setting}
We consider a symbolic optimization problem in the canonical minimization form
\begin{equation}
\begin{aligned}
\mathcal P:\quad \inf_{x\in\mathcal D}\quad & f_0(x)\\
\mathrm{s.t.}\quad & f_i(x)\le 0,\quad i=1,\ldots,m,\\
& Ax=b,
\end{aligned}
\label{eq:general_problem}
\end{equation}
where $\mathcal D$ specifies variable bounds and any integer or binary restrictions, and each $f_i$ has a finite, traceable symbolic representation. The symbolic engine and operator library support structurally admissible rewrites with explicit domain and context conditions. A finite admissible convexification trajectory produces a convex relaxation. This construction framework accommodates different expression classes through suitable operators.

In the quadratic instantiation,
\begin{equation}
f_i(x)=x^\top Q_i x+c_i^\top x+d_i,\qquad i=0,\ldots,m,
\label{eq:qcqp_instantiation}
\end{equation}
with symmetric, possibly indefinite $Q_i$. McCormick and box-based QCR actions require finite valid bounds. The fractional instantiation handles quadratic numerators over affine denominators, with certified denominator intervals excluding zero and valid auxiliary bounds. Benchmark and application formulations are given in Appendices~\ref{sec_supp_instances} and~\ref{sec_supp_applications}.

\subsection{Symbolic formulation states}
We represent a nonconvex problem by its symbolic formulation, not by a fixed coefficient-level description. The formulation is first put into a canonical form by converting the objective to minimization form and writing each inequality constraint as $f_i(x)\le 0$, with all terms collected on the left. Right-hand-side constants are absorbed into the constraint functions when needed.

To enable term-level relaxation decisions, each objective or constraint-side function is decomposed into traceable symbolic terms,
\begin{equation}
f_i(x) = \sum_{k=1}^{K_i} \tau_{i,k}(x), \ i=0,1,\cdots,m,
\label{eq:term_decomp}
\end{equation}
where $f_0$ is the objective function, $f_i$, $i\ge 1$, are the left-hand-side constraint functions after canonicalization, $K_i$ is the number of terms in $f_i$, and $\tau_{i,k}(x)$ is the $k$th symbolic term in $f_i$. Rules to obtain this decomposition, including expansion, collection of like terms, symmetrization of quadratic forms, and treatment of affine or already-convex terms, are given in Appendix~\ref{sec_supp_canonical}. Each term is retained throughout the relaxation trajectory and serves as a candidate location for relaxation actions. For each term, we store its algebraic type, sign, nesting depth, participating variables, and position in the formulation.

At relaxation step $t$, let $\mathcal{F}_t$ denote the current symbolic formulation after all previous rewrites. A valid action rewrites a selected term, updates the formulation, and refreshes the admissible action set. Relaxation design is therefore represented as a sequence of symbolic formulation states, with each term kept tied to its role in the objective or constraints.

\subsection{Tripartite graph representation of states}
Given state $\mathcal{F}_t$, we construct an augmented tripartite graph
\begin{equation}
\widetilde{G}_t = (\mathcal{V}_x, \mathcal{V}_\tau, \mathcal{V}_{\mathrm{pos}},
\mathcal{E}_{x\tau}, \mathcal{E}_{\tau\mathrm{pos}}, \mathcal{E}_{\tau\tau}^{\mathrm{aux}}),
\label{eq:trip_graph}
\end{equation}
where $\mathcal{V}_x$, $\mathcal{V}_\tau$, and $\mathcal{V}_{\mathrm{pos}}$ contain variable, term, and objective/constraint-position nodes, respectively. The corresponding nodes are denoted by $v_{x,j}$ for variable $x_j$, $v_{\tau,i,k}$ for symbolic term $\tau_{i,k}(x)$ in Eq.~\eqref{eq:term_decomp}, and $v_{\mathrm{pos},i}$ for the position of $f_i$, with $i=0$ for the objective and $i\ge 1$ for the constraints. $\mathcal{E}_{x\tau}$ and $\mathcal{E}_{\tau\mathrm{pos}}$ contain variable--term and term--position edges. $(v_{x,j},v_{\tau,i,k}) \in \mathcal{E}_{x\tau}$ links variable $x_j$ to term $\tau_{i,k}(x)$; $(v_{\tau,i,k},v_{\mathrm{pos},i})\in\mathcal{E}_{\tau\mathrm{pos}}$ links that term to the position of $f_i$. The graph therefore keeps each relaxation target tied to both its variables and its formulation context. The tripartite backbone is augmented by auxiliary $\mathcal{E}_{\tau\tau}^{\mathrm{aux}}$. It connects nested term nodes and passes local decomposition information not captured by variable membership alone (Fig.~\ref{fig:tripartite_graph}).
\begin{figure}
\centering
\includegraphics[width=0.6\linewidth]{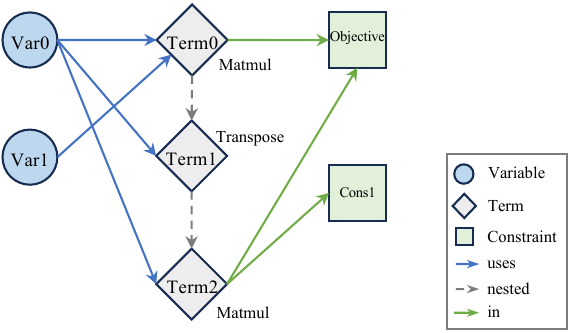}
\caption{Augmented tripartite graph representation of a symbolic formulation. Variable, term, and objective/constraint-position nodes encode the structural dependencies used for relaxation decisions.}
\label{fig:tripartite_graph}
\end{figure}

Each node carries a type-specific feature vector
\begin{equation}
h_v^{(0)} =
\begin{cases}
\phi_x(v), & v \in \mathcal{V}_x,\\
\phi_\tau(v), & v \in \mathcal{V}_\tau,\\
\phi_{\mathrm{pos}}(v), & v \in \mathcal{V}_{\mathrm{pos}},
\end{cases}
\label{eq:node_feat}
\end{equation}
where $\phi_x$ encodes variable bounds and type; $\phi_\tau$ encodes term category, sign, and structural depth; $\phi_{\mathrm{pos}}$ encodes position context, including the objective/constraint indicator and any retained constraint-sense or right-hand-side metadata. Variable--term edge attributes record local participation information, including numerator or denominator membership in fractional structures and the enclosing outer-term identifier when relevant.

A heterogeneous message-passing encoder maps the graph to contextual embeddings over the tripartite backbone and auxiliary relation. A relation-wise mean-aggregation update at layer $l+1$ takes the form
\begin{equation}
h_v^{(l+1)} =
\sigma \left(
W_{\mathrm{self}}^{(l)} h_v^{(l)}
+
\sum_{\kappa \in \mathcal{K}}
\sum_{u \in \mathcal{N}_\kappa(v)}
\frac{1}{|\mathcal{N}_\kappa(v)|}
\left(
W_{\kappa,h}^{(l)} h_u^{(l)}
+
W_{\kappa,e}^{(l)} e_{uv}
\right)
\right),
\label{eq:gnn_update}
\end{equation}where $v$ is a node in $\mathcal{V}_x\cup\mathcal{V}_\tau\cup\mathcal{V}_{\mathrm{pos}}$; $h_v^{(l)}$ is its layer-$l$ embedding; $\mathcal{K}=\{x\leftrightarrow\tau; \tau\leftrightarrow\mathrm{pos}; \tau\to\tau\}$ contains the variable--term, term--position, and auxiliary directed term--term relations; $\mathcal{N}_\kappa(v)$ denotes neighbors of $v$ under relation $\kappa$; $e_{uv}$ is the edge attribute vector, set to zero when no edge attribute is available; $W_{\mathrm{self}}^{(l)}$, $W_{\kappa,h}^{(l)}$, and $W_{\kappa,e}^{(l)}$ are learnable relation-specific transformations; $\sigma$ is a nonlinear activation.

After $L$ message-passing layers, the final embedding of node $v$ is
\begin{equation}
z_v = h_v^{(L)}.
\label{eq:node_embedding}
\end{equation}
A graph-level state embedding is obtained by pooling node embeddings separately within each node type and concatenating the summaries
\begin{equation}
z_{\widetilde{G}} = \operatorname{Concat}\Big(
\operatorname{Pool}(\{z_v : v \in \mathcal{V}_x\}),
\operatorname{Pool}(\{z_v : v \in \mathcal{V}_\tau\}),
\operatorname{Pool}(\{z_v : v \in \mathcal{V}_{\mathrm{pos}}\})
\Big).
\label{eq:graph_pool}
\end{equation}

The encoder supplies graph-level and candidate-term embeddings to the policy heads. The implementation uses relation-specific attention and graph convolutions, detailed in Appendix~\ref{sec_supp_architecture}.

\subsection{Valid relaxation actions}
A relaxation action is a validity-checked symbolic rewrite. Let $\mathcal X_0$ denote the feasible set of Eq.~\eqref{eq:general_problem}. A formulation $\mathcal F_t$ with auxiliary variables $\omega_t$, feasible set $\mathcal X_t$, and objective $f_{0,t}$ is a valid relaxation if
\begin{equation}
\forall x\in\mathcal X_0,\quad \exists\omega_t:\quad
(x,\omega_t)\in\mathcal X_t,\quad
f_{0,t}(x,\omega_t)\le f_0(x).
\label{eq:relaxation_validity}
\end{equation}
Thus, original feasible points remain representable and their minimization objective values do not increase. A convex $\mathcal F_t$ has a relaxation objective whose infimum bounds the original optimum from below. This definition applies to the complete lifted formulation, including auxiliary constraints and variable bounds, rather than to convexity of an isolated replacement expression.

At step $t$, an action is
\begin{equation}
a_t = (\ell_t,o_t), \quad
\mathrm{with} \ \ell_t\in\mathcal{L}_t=\Lambda_t\cup\{\ell_{\mathrm{glob}}\}, \ o_t\in\mathcal{O}.
\label{eq:relaxation_action}
\end{equation}
Here, $\Lambda_t={\ell_1,\ell_2,\cdots,\ell_{n_t}}$ are candidate nonconvex term locations in the current formulation; $\ell_{\mathrm{glob}}$ is the global location for formulation-level actions; $\ell_t$ is the selected target location; $o_t$ is a relaxation operator from predefined library $\mathcal{O}$. 

A target location is usually a candidate term node; formulation-level post-processing actions use a special global location. Table~\ref{tab:action_conditions} summarizes the operator library $\mathcal O$ and its main applicability conditions. Reformulation--linearization technique (RLT) constraints are generated within SDP and cut actions, rather than exposed as a separate policy action.

Each action has rule-based applicability conditions on the target structure, variable domains, and formulation context. The admissible action set is
\begin{equation}
\mathcal{A}_t=
\{(\ell,o): \ell\in\mathcal{L}_t,\ o\in\mathcal{O},\
\operatorname{Valid}(\ell,o;\mathcal{F}_t)=1\},
\label{eq:admissible_action_set}
\end{equation}
where $\mathrm{Valid}(\ell,o;\mathcal{F}_t)$ is a rule-based validity indicator. Actions outside $\mathcal{A}_t$ are masked out before policy selection.

Applying an admissible action updates the symbolic formulation deterministically:
\begin{equation}
\mathcal{F}_{t+1}=\operatorname{Rewrite}(\mathcal{F}_t,a_t).
\label{eq:rewrite_update}
\end{equation}
The $\operatorname{Rewrite}$ replaces the selected term by the surrogate generated by $o_t$ and inserts any required auxiliary variables or constraints. Term-level rewriting continues until the formulation becomes convex. The rollout may then execute a bounded number of post-convex global actions before termination.

The action library uses classical McCormick and SDP constructions \citep{mccormick1976computability,luo2010semidefinite}. The SDP action and experimental reference combine PSD lifting with supported strengthening constraints. The QCR action uses a box-based spectral under-estimator with one chord per concave eigendirection. Full constructions and their validity arguments are given together in Appendix~\ref{sec_supp_single_step}.

\begin{table}[t]
\centering
\caption{Action library and main applicability conditions. Full validity conditions are given in Appendix~\ref{sec_supp_single_step}.}
\label{tab:action_conditions}
\small
\renewcommand{\arraystretch}{1.12}
\begin{tabular}{@{}p{0.18\linewidth}p{0.34\linewidth}p{0.42\linewidth}@{}}
\toprule
Action & Construction & Main conditions \\
\midrule
Integrality relaxation & Enlarge discrete domains to intervals & Integer or binary variables; retained bounds \\
Fractional removal & Lift the ratio and its defining products & Nonzero denominator; valid auxiliary bounds \\
McCormick & Add product and square envelopes & Supported factors with finite valid bounds \\
SDP & Lift quadratic blocks with strengthening & Supported scalar quadratic forms; compatible domains \\
QCR & Under-estimate concave eigendirections & Finite bounds; correct direction at every changed occurrence \\
Bound tightening & Propagate implied bounds & Supported affine or positive square-sum constraints \\
Cut generation & Add or refresh envelopes and cuts & Valid domains, curvature, and shared representations \\
\bottomrule
\end{tabular}
\end{table}

\subsection{Policy over valid actions}
The policy selects a target location and then an operator conditioned on that target:
\begin{equation}
\pi_\theta(a_t\mid\mathcal{F}_t)
=
\pi_\theta^{(\ell)}(i_t\mid\mathcal{F}_t)
\pi_\theta^{(o)}(o_t\mid i_t,\mathcal{F}_t),
\quad
\mathrm{with} \ a_t=(\ell_t,o_t).
\label{eq:factor_policy}
\end{equation}
The location head uses graph-level and candidate-term embeddings contextualized by multi-head self-attention \citep{vaswani2017attention}. The operator head combines the selected token with its domain and structural attributes. While nonconvex terms remain, the location head selects terms admitting at least one operator in $\mathcal A_t$; after convexification, it selects the global location for bound tightening or cut generation. The conditional operator head is supported on
\begin{equation}
\mathcal{M}_t^{(o)}(i_t)=\{o\in\mathcal{O}:(\ell_t,o)\in\mathcal{A}_t\},
\label{eq:operator_mask}
\end{equation}
Both heads use masked softmax distributions, so actions outside $\mathcal A_t$ receive zero probability. This factorization combines term-specific decisions with the composition guarantee of Proposition~\ref{prop:composition}. Both heads are sampled during training and decoded greedily during evaluation. Token construction, attention, logits, and masks are detailed in Appendix~\ref{sec_supp_architecture}.

\subsection{Policy optimization with solver feedback}
The policy generates a trajectory $\zeta$ from $\mathcal F_0$, terminating at the action horizon or after the budget for post-convex global actions is exhausted. Training maximizes
\begin{equation}
J(\theta)
=
\mathbb{E}_{\zeta\sim\pi_\theta}
\left[
\sum_{t=0}^{T-1}\gamma^t r_t
+
\mathbb{I}\{\mathcal{F}_T\ \mathrm{is\ convex}\}
r_T^{\mathrm{solver}}
\right].
\label{eq:rl_objective}
\end{equation}
Here $r_t$ rewards convexification progress and structural simplification while penalizing formulation growth and ineffective rewrites; $r_T^{\mathrm{solver}}$ rewards normalized bound improvement over the cached root relaxation. Admissibility preserves validity, while solver feedback guides the choice among valid trajectories.

We optimize with proximal policy optimization \citep{schulman2017proximal}, updating the graph encoder, attention module, and policy heads jointly. Completed convex surrogates are exported to conic or semidefinite solvers, whose objectives and formulation statistics supply terminal feedback and evaluation metrics. Reward equations, PPO loss, and numerical settings are given in Appendices~\ref{sec_supp_reward}--\ref{sec_supp_training}.

\subsection{Theoretical guarantees for learned relaxation construction}\label{sec_theory}
The analysis establishes three complementary guarantees: policy-independent composition validity, unified error propagation across heterogeneous representations, and feasible recovery with an optimality-gap certificate. Together, they connect learned construction decisions to the mathematical properties of the resulting relaxation.

\begin{proposition}[Context-consistent composition]\label{prop:composition}
Let $\mathcal X_0\ne\emptyset$ and $p_0^\star=\inf_{x\in\mathcal X_0}f_0(x)$. Consider a finite trajectory satisfying the domain, occurrence-context, and representation-consistency conditions in Appendix~\ref{sec_supp_validity}. Then every $x\in\mathcal X_0$ has a feasible extension $E_T(x)\in\mathcal X_T$ with $f_{0,T}(E_T(x))\le f_0(x)$, and
\begin{equation}
L_T:=\inf_{z\in\mathcal X_T}f_{0,T}(z)\le p_0^\star.
\label{eq:composition_bound}
\end{equation}
This holds for any policy and any finite stopping time. The terminal formulation is a convex relaxation if it also satisfies the structural conditions in Appendix~\ref{sec_supp_trajectory_proof}.
\end{proposition}

\begin{theorem}[Error propagation for mixed representations]\label{thm_error}
For a terminal point $\bar z=(\bar x,\bar\omega)$, let $\phi_v$ and $\bar y_v$ be the exact operations and represented values on an acyclic graph. At each non-input node, assume finite residual bounds $\eta_v\ge|\bar y_v-\phi_v(\bar y_{\operatorname{pa}(v)})|$ and sensitivities $L_{vu}\ge0$ on boxes containing both exact and represented parent values; division boxes exclude zero. Set $\eta_v=e_v=0$ at original inputs, whose exact and represented values coincide, and recursively define
\begin{equation}
e_v=\eta_v+\sum_{u\in\operatorname{pa}(v)}L_{vu}e_u,
\qquad v\ \text{non-input}.
\label{eq:error_recursion}
\end{equation}
For output nodes $v_i$ representing $f_i$ by $\widetilde f_i$, with $\widetilde f_0=f_{0,T}$,
\begin{equation}
|f_i(\bar x)-\widetilde f_i(\bar z)|\le E_i:=e_{v_i},
\qquad i=0,\ldots,m.
\label{eq:output_error}
\end{equation}
Each $\widetilde f_i\le0$, $i\ge1$, is a retained descendant of its original inequality.
\end{theorem}

\begin{corollary}[Feasible recovery and optimality certification]\label{cor_recovery}
Under Proposition~\ref{prop:composition}, let $\bar z\in\mathcal X_T$ and let $E_i$ satisfy Theorem~\ref{thm_error}. If $\bar x$ satisfies the original domains, $A\bar x=b$, and $\widetilde f_i(\bar z)+E_i\le0$ for every original inequality, then $\bar x\in\mathcal X_0$. For any certified $d_T\le L_T$,
\begin{equation}
0\le f_0(\bar x)-p_0^\star
\le f_0(\bar x)-d_T
\le E_0+f_{0,T}(\bar z)-d_T.
\label{eq:recovery_gap}
\end{equation}
If $\bar z$ attains $L_T$ and all output discrepancies vanish, these domain and equality checks suffice for global optimality.
\end{corollary}
These guarantees allow heterogeneous relaxation mechanisms to be composed with valid bound semantics, quantify their joint representation errors, and support recovery in the original problem. The composition proof is given in Appendix~\ref{sec_supp_validity}, mixed-error bounds and their proof in Appendix~\ref{sec_supp_error}, and the recovery proof in Appendix~\ref{sec_supp_recovery}. Admissibility checks and certificate evaluation are detailed in Appendix~\ref{sec_supp_implementation}.

\section{Results}\label{sec_results}
The Results progress from aggregate performance through structure-level adaptation, term-level inspection, and transfer to fractional problems, culminating in generalization across real-world applications. We first evaluate the bound--cost trade-off of convex surrogates obtained on held-out QCQP benchmarks. We then group the benchmarks by dominant nonconvexity to test whether the relaxations adapt to problem structure. Individual trajectories are further inspected to show how term-level rewrites are arranged inside a formulation. Fractional instances next test whether the learned policy transfers beyond polynomial formulations. Application-derived problems finally test generalization beyond the benchmark setting.
\subsection{Experimental settings}\label{sec_experimental_settings}
\paragraph{Benchmarks}
The experiments evaluate the general construction methodology through QCQP and fractional instantiations. The synthetic QCQP benchmarks span mixed, dense, separable, bilinear, binary-coupled, concave-objective, and nonconvex-constraint structures. We report results on 241 held-out QCQPs and 15 held-out fractional instances. Detailed instance construction is given in Appendix~\ref{sec_supp_instances}.

\paragraph{Comparison methods and selection rationale}
The comparison methods are selected to evaluate \emph{relaxation-construction capabilities} and the \emph{contribution of learned decisions}. As analyzed in Section~\ref{sec_related_work}, the related learning methods make decisions within prescribed schemes or perform inference-time generation and verification; they do not implement the \emph{joint symbolic location--operator construction process} studied here.

McCormick and strengthened SDP serve as \emph{classical reference formulations with different strength--size trade-offs}. They are constructed by the same symbolic engine, with common preprocessing. McCormick provides a lightweight envelope-based reference. Strengthened SDP provides a stronger matrix-based reference. These references test whether learned composition \emph{strengthens lightweight relaxations} and \emph{approaches stronger bounds with selective formulation growth}.

Random and Heuristic isolate the \emph{value of learning under the same action library and legality masks}. Random samples legal locations and operators. Heuristic follows a fixed term order, applies McCormick to approximately separable terms, and selects SDP for sufficiently large and densely coupled quadratic terms. These baselines test whether the supplied operators and validity checks alone explain the results, and whether learned location--operator decisions improve on fixed construction rules.

Gurobi and SCIP root relaxations provide \emph{solver-level starting bounds before branching}. They test whether the generated surrogate strengthens the formulation relative to these starting points, without attributing gains to subsequent branch-and-bound search. Taken together, the comparisons evaluate classical reference strength, the contribution of learned construction, and solver-level starting bounds.

Training settings, detailed baseline rules, evaluation metrics, and the solver protocol are provided in Appendices~\ref{sec_supp_training}, \ref{sec_supp_baselines}, \ref{sec_supp_metrics}, and~\ref{sec_supp_protocol}, respectively.

\subsection{Learned trajectories improve relaxation bound--cost trade-off}
We first test whether learned relaxation trajectories improve the bound--cost trade-off. The policy is trained on synthetic QCQP benchmarks and tested on held-out instances. The benchmarks span mixed, dense, separable, bilinear, binary-coupled, concave-objective, and nonconvex-constraint structures. Existing learning-based methods do not provide direct counterparts because they refine predefined schemes rather than generate term--operator rewrite trajectories. The comparison therefore uses standard references with different bound strength and formulation size: Gurobi and SCIP root relaxations provide solver-level references; McCormick envelopes provide a lightweight handcrafted endpoint; semidefinite relaxation (SDP) with bound-derived strengthening provides a strong but larger handcrafted endpoint. We also compare with two non-learned baselines over the same action library: the Random baseline selects legal locations and operators at random; the Heuristic baseline follows a fixed term order and selects operators by term type, size, and coupling density. Performance is measured by lower-bound quality and formulation growth. Details of benchmark construction and experimental settings are provided in Appendices~\ref{sec_supp_instances} and~\ref{sec_supp_settings}.

The learned trajectories reach convex terminal formulations on all 241 held-out instances within the rollout and improve solver-level starting bounds (Table~\ref{tab:solver_root}, Fig.~\ref{fig:root_bound_improvement}). Using the Gurobi root relaxation as the reference, the learned surrogates improve the mean lower bound by 21.88\% and give tighter bounds on 53.9\% of the test instances. Using the SCIP root relaxation as the reference, the mean improvement is 15.37\%, with tighter bounds on 48.5\% of the instances. These root relaxations are the convex starting points used by general-purpose solvers before branching. The gains therefore show that the policy strengthens the formulation itself, rather than relying on downstream tree search. The ranked distributions further show that these gains occur across a substantial subset of instances, not only in a few outliers.
\begin{table}[t]
\centering
\caption{Bound improvement of learned trajectories over solver root relaxations. The reported lower-bound improvements are computed as instance-wise normalized differences and averaged over 241 held-out instances. Tighter/tie/weaker reports the fraction of instances on which the learned trajectory gives a tighter, equal, or weaker lower bound than the corresponding solver root relaxation.}
\label{tab:solver_root}
\renewcommand{\arraystretch}{1.15}
\begin{tabular*}{\columnwidth}{@{\extracolsep{\fill}}ccc@{}}
\toprule
\makecell[c]{Solver root relaxation} &
\makecell[c]{Lower-bound improvement (\%)} &
\makecell[c]{Tighter / Tie / Weaker (\%)} \\
\midrule
Gurobi & +21.88 & 53.9 / 29.5 / 16.6 \\
SCIP   & +15.37 & 48.5 / 14.0 / 37.4 \\
\bottomrule
\end{tabular*}
\end{table}

\begin{figure}[t]
    \centering
    \begin{minipage}{0.47\textwidth}
        \centering
        \textbf{a}\par\vspace{2pt}
        \includegraphics[width=\linewidth]{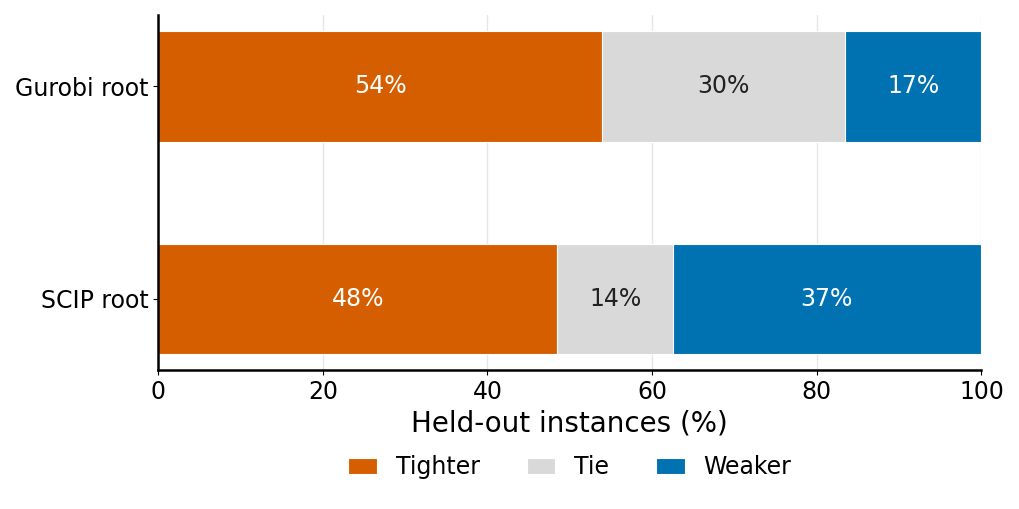}
    \end{minipage}
    \hfill
    \begin{minipage}{0.47\textwidth}
        \centering
        \textbf{b}\par\vspace{2pt}
        \includegraphics[width=\linewidth]{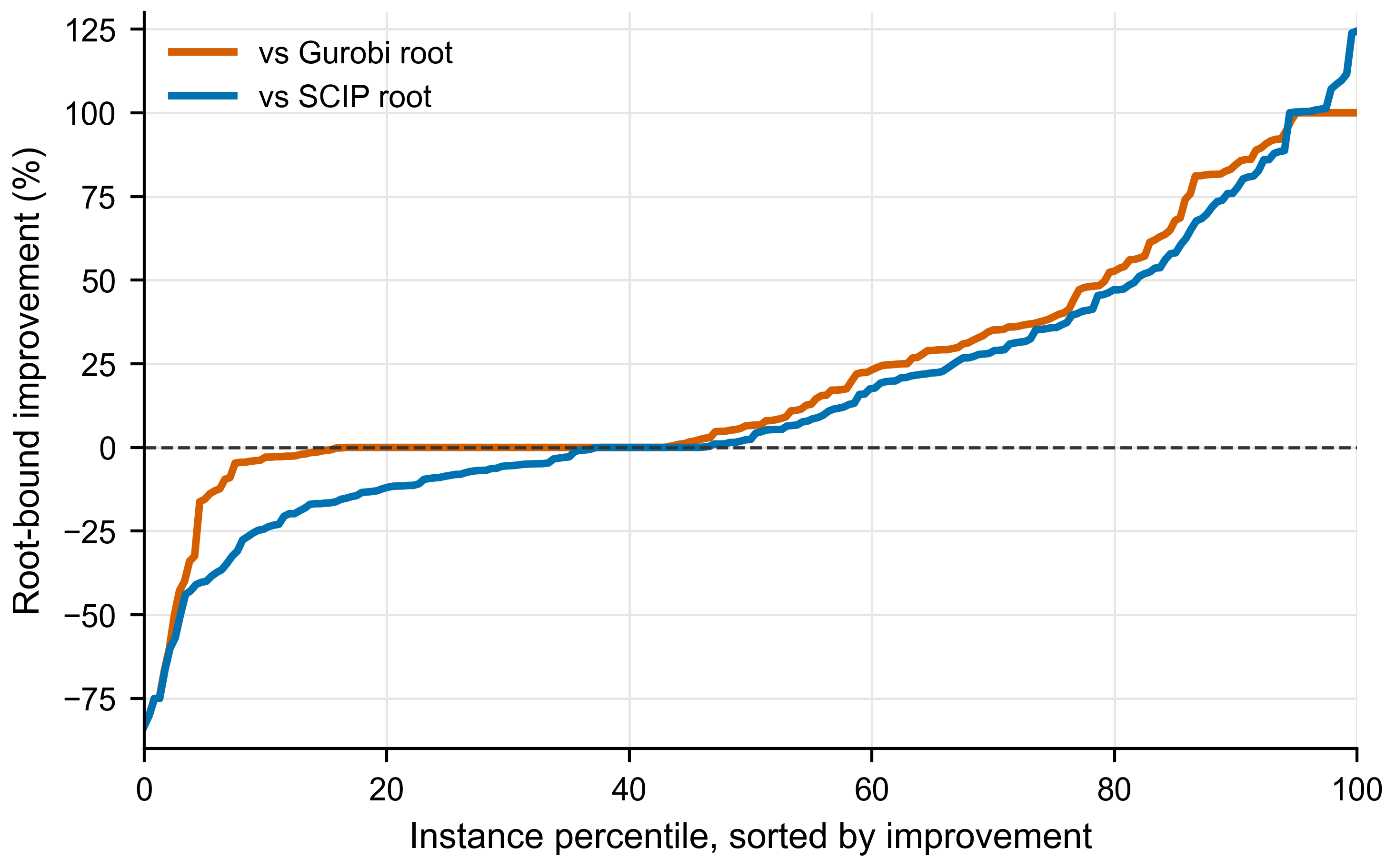}
    \end{minipage}
\caption{\textbf{a}, Pairwise comparison with Gurobi and SCIP root relaxations. Each bar reports the percentage of test instances on which the learned relaxation gives a tighter, equal, or weaker lower bound than the corresponding solver root relaxation. \textbf{b}, Ranked lower-bound improvement over solver root relaxations. 
    Test instances are sorted by the percentage improvement of the learned lower bound over the corresponding solver root bound. Positive values indicate tighter learned relaxations.}
    \label{fig:root_bound_improvement}
\end{figure}

The learned trajectories further recover much of the strength of SDP without paying its full formulation cost (Table~\ref{tab:main_tradeoff}). McCormick is small but loose. The learned trajectories improve the mean lower bound over McCormick by 13.52\% and give tighter bounds on 47.3\% of instances, although this requires larger formulations. This confirms that lightweight envelopes leave substantial bound quality unused. The comparison with SDP is more revealing. The learned trajectories have only a 3.70\% mean lower-bound gap to SDP, with ties on 74.3\% of instances. Yet they use 34.90\% fewer variables, 29.52\% fewer constraints, and 28.79\% fewer nonzeros than SDP. Thus, most of the SDP-level bound can be recovered without constructing the full SDP relaxation. The learned trajectories make formulation growth selective and effective, rather than simply making the relaxation larger.
\begin{table}[t]
\centering
\caption{Bound quality and formulation growth of learned trajectories over reference and baseline relaxations. $\Delta$Vars, $\Delta$Cons, and $\Delta$NNZ are formulation-growth metrics. They report relative changes in the numbers of variables, constraints, and nonzeros of the learned formulation with respect to the reference formulation. Positive values indicate a larger learned formulation; negative values indicate smaller.}
\label{tab:main_tradeoff}
\renewcommand{\arraystretch}{1.15}
\begin{tabular*}{\textwidth}{@{\extracolsep{\fill}}cccccc@{}}
\toprule
\makecell[c]{Reference \& baseline \\relaxation} &
\makecell[c]{Lower-bound\\improvement\\ (\%)} &
\makecell[c]{Tighter / Tie / Weaker\\(\%)} &
\makecell[c]{$\Delta$Vars\\(\%)} &
\makecell[c]{$\Delta$Cons\\(\%)} &
\makecell[c]{$\Delta$NNZ\\(\%)} \\
\midrule
McCormick & +13.52 & 47.3 / 44.4 / 8.3 & +190.97 & +274.88 & +252.71 \\
SDP       & -3.70 & 1.2 / 74.3 / 24.5 & -34.90 & -29.52 & -28.79 \\
Random    & +17.68 & 68.5 / 27.4 / 4.1 & +112.87 & +211.42 & +191.06 \\
Heuristic & +13.80 & 49.8 / 43.6 / 6.6 & +190.97 & +310.32 & +272.61 \\
\bottomrule
\end{tabular*}
\end{table}

Finally, the learned trajectories outperform both non-learned baselines under the same action library and legality masks (Table~\ref{tab:main_tradeoff}). Compared with Random, they improve the mean lower bound by 17.68\% and are tighter or tied on 95.9\% of the instances. This result shows that the masks only enforce validity and do not determine trajectory quality. Compared with Heuristic, learned trajectories improve the mean lower bound by 13.80\% and are tighter or tied on 93.4\% of the instances. Thus, Heuristic's fixed term--operator rules and processing orders do not recover the learned behavior. The gains therefore come from learning which locations and operators to select and in what order.

\subsection{Relaxation choices adapt to problem structure}
We next test whether the policy adapts its relaxation choices to QCQP structures and whether this adaptation yields useful bound--cost trade-offs. Held-out instances are grouped by their dominant source of nonconvexity. Within each group, accepted relaxation actions reveal whether the policy prefers different choices for different structures. Group-level bound improvement and formulation growth then show whether these structure-dependent choices lead to favorable trade-offs.

Fig.~\ref{fig:family_action_composition} reveals clear structural adaptation. The learned trajectories do not simply default to the strongest operator. SDP lifting appears in all families, but its intensity is modulated by problem structure. For concave quadratic objectives, SDP actions dominate, matching the need to handle nonconvex curvature through lifting. For mixed-structure QCQPs, dense quadratic terms, diagonal or separable quadratics, and binary-coupled quadratics, the policy combines SDP lifting with McCormick relaxations, using local envelope construction together with stronger lifted reformulations. Integrality relaxation is concentrated in families with discrete or domain-induced structure, whereas QCR is invoked only in selected regimes. These operator allocations are the key evidence of adaptation. The policy composes lifting, envelope construction, domain relaxation, and reformulation according to the source of nonconvexity, rather than committing to one fixed relaxation template.
\begin{figure}
    \centering
    \includegraphics[width=0.6\linewidth]{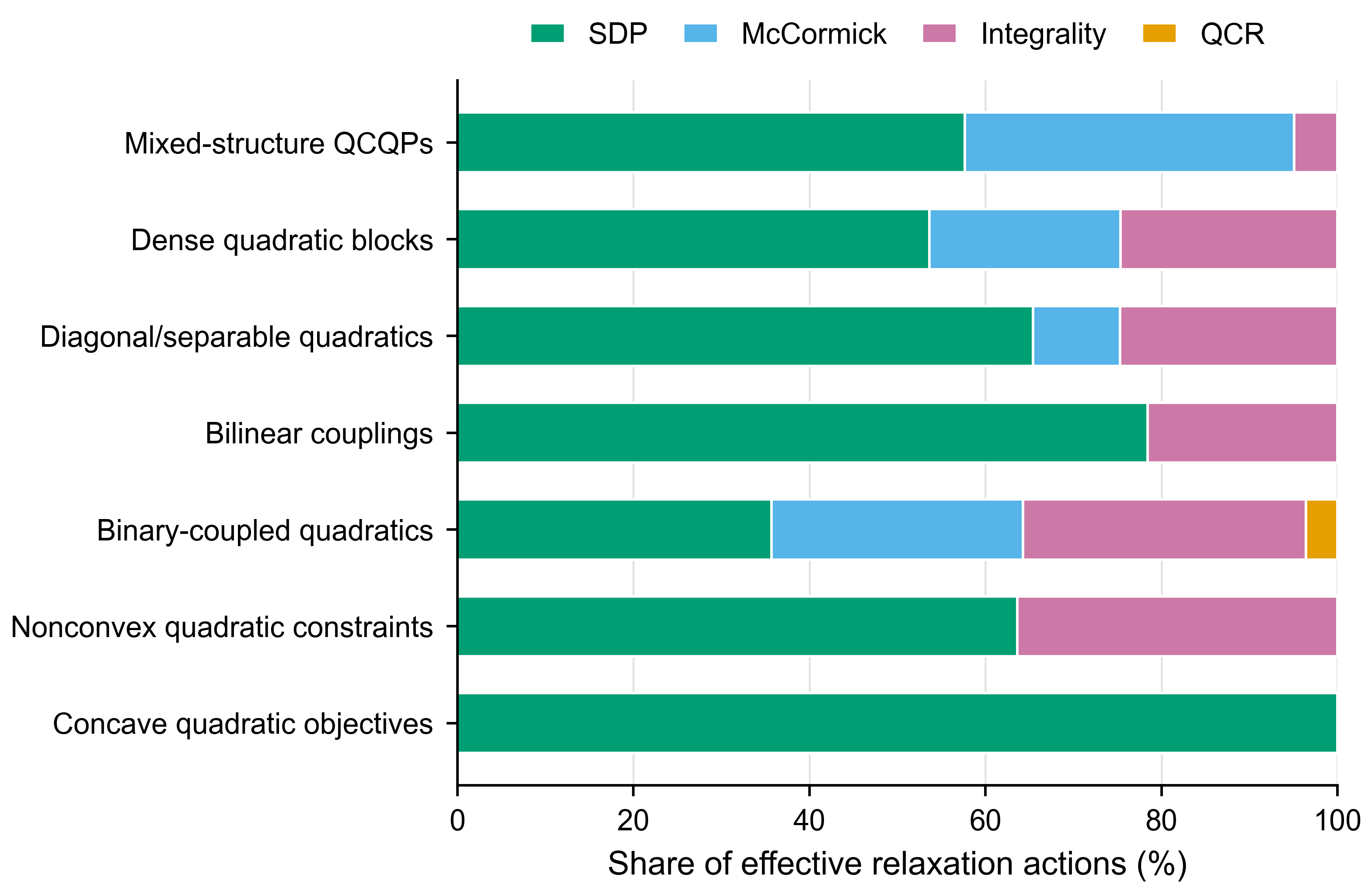}
    \caption{Relaxation action composition in learned trajectories across QCQP structural families. Only accepted actions that changed the formulation are counted.}
    \label{fig:family_action_composition}
\end{figure}

Further structure-dependent bound and cost analyses are provided in Appendix~\ref{sec_supp_structure_details}.

\subsection{Learned trajectories place and sequence relaxation mechanisms}
We further inspect individual trajectories to understand how structure adaptation is implemented within a formulation. The analysis focuses on two degrees of freedom that are invisible in aggregate tables. The first is placement: which nonconvex term receives which relaxation mechanism. The second is sequencing: how these term-level rewrites are ordered to form the final convex surrogate. Three held-out QCQP instances with heterogeneous nonconvex structure are selected for inspection (Table~\ref{tab:representative_trajectories}). \texttt{H\_MIXED\_5\_n3\_4\_615} illustrates placement, where different quadratic terms receive different operators. \texttt{E\_EXP\_P3\_n4\_945} illustrates sequencing, where domain relaxation, envelope construction, and lifting are applied in order. \texttt{H\_MIXED\_1\_n4\_929} is analyzed in Appendix~\ref{sec_supp_trajectory3}, as a lower-return case that still exposes term-level selectivity.
\begin{table}[t]
\centering
\caption{Lower bounds of held-out instances for term-level trajectory inspection. Root LB, AutoConvexRelax LB, McC. LB, and SDP LB are short for lower bounds of solver root relaxation, learned trajectory, McCormick, and SDP, respectively.}
\label{tab:representative_trajectories}
\small
\renewcommand{\arraystretch}{1.18}
\setlength{\tabcolsep}{4pt}
\begin{tabular*}{\linewidth}{@{\extracolsep{\fill}}ccccc@{}}
\toprule
\makecell[c]{Instance} &
\makecell[c]{Root LB} &
\makecell[c]{AutoConvexRelax LB} &
\makecell[c]{McC. LB} &
\makecell[c]{SDP LB} \\
\midrule
\makecell[c]{\texttt{H\_MIXED\_5 \_n3\_4\_615}} &
$-29.73$ &
$-11.60$ &
$-29.73$ &
$-11.30$ \\
\makecell[c]{\texttt{E\_EXP\_P3 \_n4\_945}} &
$-246.44$ &
$-34.28$ &
$-35.06$ &
$-34.28$ \\
\makecell[c]{\texttt{H\_MIXED\_1 \_n4\_929}} &
$-84.40$ &
$-84.82$ &
$-85.65$ &
$-84.40$ \\
\bottomrule
\end{tabular*}
\end{table}

Detailed term-level trajectory analyses are provided in Appendix~\ref{sec_supp_trajectory_analysis}.

\subsection{The learned policy transfers to fractional QCQPs}
We next test whether relaxation knowledge learned on standard QCQPs can be transferred to fractional formulations. The policy is initialized from the standard QCQP checkpoint, fine-tuned on a fractional training split, and evaluated on held-out fractional instances. Each fractional instance combines affine denominators with fixed coefficients and nonconvex quadratic terms in the numerator and elsewhere in the formulation. Details of the instance construction are given in Appendix~\ref{sec_supp_fractional}.

The fractional results confirm that the learned policy is not tied to standard QCQPs (Table~\ref{tab:fraction_hard_extension}). All learned trajectories produce valid convex formulations with finite lower bounds. Relative to McCormick and Heuristic, they improve the mean lower bound by 11.11\% and are never weaker. The learned bounds are tighter on 8 instances and tied on 7, with formulation growth of only 3.83--10.13\%. Relative to Random, the mean lower-bound improvement is 9.05\%, with tighter bounds on 5 instances, ties on 7, and weaker bounds on 3. Against SDP, learned trajectories tie on 7 instances and have a mean lower-bound gap of 7.52\%, while producing formulations that are 20.74--28.32\% smaller. The extension therefore preserves the bound--cost pattern observed on standard QCQPs: stronger than lightweight relaxations and smaller than full SDP. More broadly, it shows that learned relaxation design can operate across symbolic expression classes rather than remain tied to a single algebraic form.
\begin{table}[t]
\centering
\caption{Bound quality and formulation growth of learned trajectories over reference and baseline relaxations on 15 held-out nonconvex fractional instances.}
\label{tab:fraction_hard_extension}
\small
\renewcommand{\arraystretch}{1.15}
\setlength{\tabcolsep}{5pt}
\begin{tabular*}{\linewidth}{@{\extracolsep{\fill}}cccccc@{}}
\toprule
\makecell[c]{Reference \& baseline\\relaxation} &
\makecell[c]{Lower-bound\\improvement\\(\%)} &
\makecell[c]{Tighter / Tie / Weaker\\(\%)} &
\makecell[c]{$\Delta$Vars\\(\%)} &
\makecell[c]{$\Delta$Cons\\(\%)} &
\makecell[c]{$\Delta$NNZ\\(\%)} \\
\midrule
McCormick & +11.11 & 53.3 / 46.7 / 0.0 & +8.82 & +10.13 & +3.83 \\
SDP & -7.52 & 0.0 / 46.7 / 53.3 & -28.32 & -21.89 & -20.74 \\
Random & +9.05 & 33.3 / 46.7 / 20.0 & +9.37 & +24.25 & +23.81 \\
Heuristic & +11.11 & 53.3 / 46.7 / 0.0 & +8.82 & +10.13 & +3.83 \\
\bottomrule
\end{tabular*}
\end{table}

\subsection{Application-derived QCQPs validate generalization beyond benchmarks}
We finally evaluate whether the learned policy generalizes beyond benchmarks. The standard-QCQP policy is frozen and applied to the application problems without retraining, fine-tuning, or task-specific reward adjustment. Two representative problems, multicast beamforming \citep{di2020hybrid} and bounded-noise sensor network localization \citep{biswas2006semidefinite,so2006theory}, are used. They introduce nonconvexity through physical signal coupling and geometric distance constraints, rather than controlled benchmark construction. Beamforming leads to dense quadratic forms arising from transmit power, interference, and quality-of-service requirements. Localization produces coupled quadratic constraints from noisy Euclidean distance measurements. Together, they provide two complementary stress tests for relaxation design under real application structure. Full formulations, parameters, and reference solution settings are provided in Appendices~\ref{sec_supp_beamforming}, \ref{sec_supp_localization}, and~\ref{sec_supp_protocol}.

Table~\ref{tab:real_applications_instance_breakdown} reports results on four beamforming and four localization instances. On beamforming, the learned relaxation achieves near-exact performance. Across B1--B4, the relative gap to the reference solution is on the order of $10^{-6}$ to $10^{-5}$, and the learned values are indistinguishable from the SDP relaxation. This indicates strong transfer from benchmark QCQPs to application-derived dense quadratic structures. Despite arising from transmit power, interference, and signal-quality constraints, beamforming induces dense quadratic coupling that differs from benchmark construction. The learned policy preserves SDP-level tightness under a clear domain shift, without retraining or any task-specific tuning.
\begin{table}[t]
\centering
\caption{Bound quality on application-derived QCQPs. B1--B4 are the maximization beamforming instances; the reported values are upper bounds. S1--S4 are the minimization localization instances; the reported values are lower bounds. Reference values are the best objective values returned by the solvers; AutoConvexRelax values are convex surrogate objectives induced by the learned policy; AutoConvexRelax gap reports the relative error to the reference.}
\label{tab:real_applications_instance_breakdown}
\centering
\scriptsize
\renewcommand{\arraystretch}{1.10}
\setlength{\tabcolsep}{4pt}
\begin{tabular*}{\textwidth}{@{\extracolsep{\fill}}lccccc@{}}
\toprule
\makecell[c]{Inst.} &
\makecell[c]{Reference} &
\makecell[c]{AutoConvexRelax} &
\makecell[c]{McCormick} &
\makecell[c]{SDP} &
\makecell[c]{AutoConvexRelax gap} \\
\midrule
B1 & 8.57 & 8.57 & 8.58 & 8.57 & $3.38 \times 10^{-6}$ \\
B2 & 2.76 & 2.76 & 3.23 & 2.76 & $3.68 \times 10^{-5}$ \\
B3 & 1.35 & 1.35 & 1.35 & 1.35 & $1.02 \times 10^{-6}$ \\
B4 & 0.34 & 0.34 & 0.34 & 0.34 & $2.91 \times 10^{-6}$ \\
\midrule
S1 & 6.99 & 6.81 & 6.25 & 6.81 & $2.71 \times 10^{-2}$ \\
S2 & 7.00 & 6.86 & 6.23 & 6.86 & $2.00 \times 10^{-2}$ \\
S3 & 6.94 & 6.78 & 6.28 & 6.78 & $2.29 \times 10^{-2}$ \\
S4 & 7.04 & 6.87 & 6.36 & 6.87 & $2.47 \times 10^{-2}$ \\
\bottomrule
\end{tabular*}
\end{table}

The sensor network localization instances are more challenging. Across S1--S4, the learned relaxation reduces the McCormick gap from 10.22\% to 2.37\% on average, while consistently improving all instances. Unlike beamforming, which is dominated by dense quadratic coupling, localization is driven by geometric distance constraints, leading to a different nonconvex structure. McCormick relaxations capture part of this geometry but remain insufficient to fully tighten the coupled constraints. The learned policy provides a consistently tighter relaxation than McCormick and matches SDP at the reported precision. This result shows that the learned policy generalizes beyond algebraic coupling to geometry-induced nonconvexity, while maintaining systematic improvement over classical envelope relaxations. 

\section{Conclusion}\label{sec_discussion}
AutoConvexRelax makes convex relaxation design an executable, learnable construction process. A tripartite symbolic representation and factorized policy select where and how to apply classical relaxation operators, composing them into problem-specific trajectories. The resulting method adapts relaxation strength to individual terms and to changes introduced by earlier rewrites.

The theoretical analysis establishes policy-independent validity for admissible compositions, unified mixed-error bounds, and sufficient conditions for feasible recovery and optimality certification. Experiments show that learned construction strengthens lightweight relaxations and approaches SDP bounds with smaller formulations. Structure-level and term-level analyses explain how selective placement and sequencing achieve this trade-off; fractional and application-derived instances demonstrate transfer beyond standard QCQPs.

Future work will expand the validity-checked operator library to broader factorable, polynomial, and mixed-integer nonlinear problems, and incorporate downstream solver feedback such as numerical stability, tree size, and wall-clock time.

\section*{Funding}
This study is supported by the National Natural Science Foundation of China (NSFC) under Grant No.~72401122, Guangdong Basic and Applied Basic Research Foundation under Grants No.~2024A1515012241, 2025A1515110439, and 2021A1515110024, Guangdong University Young Innovative Talents Program Project under Grant No.~2024KQNCX050, Shenzhen Fundamental Research Program under Grant No.~JCYJ20200109141235597, and the NSFC under Grant No.~61761136008.

\appendix
\makeatletter
% Prefix only top-level appendix headings; keep references and equation numbers unchanged.
\renewcommand{\@seccntformat}[1]{%
  \ifstrequal{#1}{section}{Appendix~}{}%
  \csname the#1\endcsname.\quad}
\makeatother
\renewcommand{\thesection}{\Alph{section}}
\numberwithin{equation}{section}
\numberwithin{table}{section}
\numberwithin{figure}{section}

\section{Canonicalization rules for symbolic terms}\label{sec_supp_canonical}
The symbolic formulation used by the action system is obtained from the original algebraic expression by deterministic canonicalization. This preprocessing defines the candidate term locations used in the action representation in the main text.

First, expressions are expanded into sums of scalar terms. Products of affine expressions are distributed so that quadratic contributions are exposed explicitly, while constants are folded into the constant offset. Second, like terms are collected. Monomials with the same variable support and degree are merged by summing their coefficients; zero-coefficient terms are removed. This step avoids presenting duplicate algebraic copies of the same nonconvex structure to the policy.

Quadratic forms are represented in a symmetric convention. A quadratic expression is collected as \(x^\top Qx+q^\top x+c\), and nonsymmetric coefficient matrices are replaced by their symmetric part \((Q+Q^\top)/2\), since \(x^\top Qx=x^\top ((Q+Q^\top)/2)x\). Cross terms are therefore stored once in the canonical quadratic block rather than as order-dependent products.

Every canonical term is retained as a traceable candidate location, as described in the main paper. Candidate locations are distinguished from currently admissible actions: affine terms and terms already convex in their objective or constraint context remain in the formulation and graph representation, but need not be relaxed. At each step, the nonconvex locations and their admissible operators are selected from this representation according to the action conditions in the main paper.

\section{Context-consistent composition: constructions and proof}\label{sec_supp_validity}
We prove Proposition~\ref{prop:composition} in exact arithmetic for the quadratic and fractional library, with $\mathcal X_0\ne\emptyset$. Construction bounds contain all recursively assigned original feasible points; recipes, substitutions, tightening, and cuts satisfy the conditions below. Preprocessing preserves these properties.

\subsection{Semantic recipes and occurrence contexts}
A \emph{recipe} specifies an auxiliary's original-feasible-point value, e.g., $w=uv$, $Z=uu^\top$, or $\lambda=p(u)/d(u)$, without constraining every relaxed point. Recipes form an acyclic graph; denominator intervals exclude zero. Reuse and equality links require identical defining functions and parents. Normalization, aliases, copying, and rewrites preserve these meanings; rewritten defining constraints undergo the same validity checks.

\begin{lemma}[Occurrence-sensitive substitution]\label{lem_context}
On a valid domain, an exact representation $\widehat q(E(x))=q(E(x))$ preserves witness values at every well-defined occurrence. An under-estimator $\widehat q\le q$ safely replaces $q$ in a minimization objective or canonical $\le0$ constraint $cq+r$ when $c\ge0$, but generally not when $c<0$. More generally, the occurrence context must be nondecreasing in the replaced argument. This condition alone does not justify one-sided substitutions in equalities or PSD entries.
\end{lemma}
\begin{proof}
Exact substitution preserves witness values. For one-sided affine substitution,
\[
(c\widehat q+r)-(cq+r)=c(\widehat q-q)\le0\quad(c\ge0),
\]
preserving objective and inequality directions. For $c<0$, the sign reverses unless the estimate is exact. Nondecreasing contexts preserve order; equalities require equal values and PSD constraints require matrix order. Splitting equalities still requires both directions.
\end{proof}

\paragraph{Why the occurrence check is necessary}
For $q(x)=-x_1^2-2x_2^2$ on $[-1,1]^2$, the problem $\min q(x)$ subject to $-q(x)\le0$ has the unique feasible point $0$. The under-estimator $\widehat q=-3$ is valid in the objective, but its substitution into the constraint yields $3\le0$, destroying feasibility.

\begin{lemma}[Simultaneous exact assignment]\label{lem_witness}
For an original feasible $x$, acyclic recipes with well-defined operations and consistent reuse admit a simultaneous topological assignment satisfying all equalities between representations of the same recipe. Fresh nodes leave previous assignments unchanged.
\end{lemma}
\begin{proof}
Starting from $x$, evaluate recipes on assigned parents in topological order. Acyclicity ensures well-defined evaluation; identical recipes and parents yield identical values, unchanged by reuse or fresh extensions. Compatible proxies therefore satisfy $w_{ij}=u_i u_j=Z_{ij}$ simultaneously.
\end{proof}

\subsection{Verification of the seven action categories}\label{sec_supp_single_step}
Each action retains or extends the common assignment $E_t(x)$ of Lemma~\ref{lem_witness}, for every $x\in\mathcal X_0$. Local error bounds also require the candidate relaxation point to satisfy the construction bounds.

\paragraph{Integrality relaxation}
Replacing $\mathbb Z\cap[\ell_j,u_j]$ by $[\ell_j,u_j]$, or $\{0,1\}$ by $[0,1]$, retains the original coordinates and all other constraints without changing expression values.

\paragraph{McCormick products and squares}
For $u\in[\ell_u,u_u]$ and $v\in[\ell_v,u_v]$, replace the product equality $w=uv$ by the McCormick envelopes
\begin{equation}
\begin{aligned}
w&\ge \ell_u v+\ell_v u-\ell_u\ell_v,&\quad
w&\ge u_u v+u_v u-u_u u_v,\\
w&\le u_u v+\ell_v u-u_u\ell_v,&
w&\le \ell_u v+u_v u-\ell_u u_v.
\end{aligned}
\label{eq:mccormick_action}
\end{equation}
With $w=uv$, these follow from the nonnegative products $(u-\ell_u)(v-\ell_v)$, $(u_u-u)(u_v-v)$, $(u_u-u)(v-\ell_v)$, and $(u-\ell_u)(u_v-v)$, respectively.

Corner-product extrema bound $w$. Square proxies satisfy endpoint tangent lower bounds and the secant upper bound,
\begin{equation}
\begin{aligned}
w&\ge 2\ell_u u-\ell_u^2,\qquad w\ge 2u_u u-u_u^2,\\
w&\le(\ell_u+u_u)u-\ell_u u_u,
\end{aligned}
\label{eq:square_action}
\end{equation}
Additional tangents are valid: at $w=u^2$, $u^2-(2au-a^2)=(u-a)^2\ge0$ and $(\ell_u+u_u)u-\ell_u u_u-u^2=(u-\ell_u)(u_u-u)\ge0$. Nonnegativity and interval bounds hold. Exact affine-factor assignments extend the argument to inner products and quadratic decompositions, including signed coefficients.

\paragraph{SDP lifting and strengthening}
For a supported quadratic block $u^\top Qu$, replace it by $\operatorname{Tr}(QZ)$ and impose
\begin{equation}
\begin{bmatrix}1&u^\top\\u&Z\end{bmatrix}\succeq0.
\label{eq:sdp_action}
\end{equation}
For supported continuous scalar quadratic blocks excluding unsafe fractional structures, assign $Z=uu^\top$. Then Eq.~\eqref{eq:sdp_action} equals $[1;u][1;u]^\top\succeq0$, $\operatorname{Tr}(QZ)=u^\top Qu$, and $\operatorname{Tr}(Z)=\|u\|^2$, establishing symmetry, product bounds, McCormick constraints, and lifted norm constraints. Concatenated affine vectors satisfy the same identities. RLT validity follows below. Overlapping blocks share coordinates; linked entries require matching product recipes. Rank one is required only at the witness.

\paragraph{Single-chord spectral QCR}
Absorb the signed coefficient into symmetric $Q$ and write $Q=Q_++\sum_{\lambda_k<0}\lambda_k v_kv_k^\top$, where $Q_+\succeq0$. Put $t_k=v_k^\top u$. Interval evaluation gives $a_k\le t_k\le b_k$. Retain convex curvature and replace each concave eigendirection by its chord:
\begin{equation}
\widehat q(u)=u^\top Q_+u+
\sum_{\lambda_k<0}\lambda_k\big[(a_k+b_k)v_k^\top u-a_kb_k\big]
\le u^\top Qu.
\label{eq:qcr_action}
\end{equation}
The exact identity
\begin{equation}
q(u)-\widehat q(u)
=\sum_{\lambda_k<0}|\lambda_k|(t_k-a_k)(b_k-t_k)\ge0
\label{eq:qcr_exact_error}
\end{equation}
proves Eq.~\eqref{eq:qcr_action}; $Q_+\succeq0$ ensures convexity. Every changed occurrence must satisfy Lemma~\ref{lem_context}, preserving existing recipes. Each chord spans its complete certified interval; maxima of subinterval chords are excluded.

\paragraph{Fractional removal}
For a nonzero constant denominator, algebraic division preserves the value. Otherwise let $p\in[p_L,p_U]$ and $d\in[d_L,d_U]$, where $0\notin[d_L,d_U]$. The ratio belongs to the interval with endpoints
\begin{equation}
\lambda_L=\min_{a\in\{p_L,p_U\},\,b\in\{d_L,d_U\}}a/b,
\qquad
\lambda_U=\max_{a\in\{p_L,p_U\},\,b\in\{d_L,d_U\}}a/b.
\label{eq:ratio_interval}
\end{equation}
Affine dependence on $p$ and monotonicity in fixed-sign $d$ give these extrema. Introduce

\begin{equation}
\lambda d=p,\qquad \lambda_L\le\lambda\le\lambda_U,
\label{eq:fraction_action}
\end{equation}
and apply supported McCormick constructions to generated products. Exact assignments $\lambda=p/d$ and product proxies preserve the defining equality, fraction value, and earlier constraints. For $d=a^\top u+b$, the products are $\lambda u_j$ and those in $p$. One-sided replacements require Lemma~\ref{lem_context}; arbitrary finite big-$M$ values cannot replace missing bounds or zero-crossing intervals.

\paragraph{Bound propagation}
For $\sum_j a_jy_j\le b$, define $m_{-i}=\sum_{j\ne i}\min\{a_j\ell_j,a_ju_j\}$. Then $a_i y_i\le b-m_{-i}$ bounds $y_i$ above for $a_i>0$ and below for $a_i<0$ by $(b-m_{-i})/a_i$; process equalities in both directions. For $\sum_j a_jy_j^2\le R$ with $a_j\ge0$, subtract $s_j=\min_{[\ell_j,u_j]}y_j^2$ times $a_j$ for $j\ne i$ to obtain, when $a_i>0$,

\begin{equation}
|y_i|\le \sqrt{\frac{R-\sum_{j\ne i}a_j s_j}{a_i}}.
\label{eq:ssq_bound}
\end{equation}
The numerator is nonnegative at any witness. Under the stated expression and coefficient conditions, finite iterations of interval intersection preserve every witness.

\paragraph{Cut generation}
Refreshing envelopes with new valid bounds retains exact products. For an affine slack $g(y)=b-a^\top y\ge0$ and $y_k\in[\ell_k,u_k]$, the RLT products
\begin{equation}
g(y)(y_k-\ell_k)\ge0,
\qquad g(y)(u_k-y_k)\ge0
\label{eq:rlt_validity}
\end{equation}
hold at every witness. Replacing products by compatible proxies preserves their values. For a current constraint $h(y)\le b$ with differentiable convex $h$, the tangent inequality
\[
h(y_0)+\nabla h(y_0)^\top(y-y_0)\le h(y)\le b
\]
holds at every witness even for infeasible $y_0$. For $h(y)=\alpha y^\top Qy$, require $\alpha Q\succeq0$. Cuts may exclude other relaxed points.

\subsection{Complete trajectory proof and terminal convexity}\label{sec_supp_trajectory_proof}
The terminal structural conditions are continuous convex variable domains, a convex objective and convex canonical inequality functions, affine equalities, and PSD constraints given by affine matrix maps.

\begin{proof}[Proof of Proposition~\ref{prop:composition}]
Fix $x\in\mathcal X_0$, $E_0(x)=x$. Maintain three simultaneous invariants: recipe-valued auxiliaries, satisfied constraints and bounds, and objective at most $f_0(x)$, also justifying subsequent cuts and links.

Assume the invariants at $t$. Integrality relaxation retains coordinates; exact lifting extends recipes by Lemma~\ref{lem_witness}, satisfies new constraints by Appendix~\ref{sec_supp_single_step}, and preserves occurrence values. Lemma~\ref{lem_context} handles one-sided scalar substitutions; equality and PSD changes require separate justification. The bound and cut deductions apply at the same assignment. Rejected proposals and exact normalization preserve the invariants, completing induction.

Thus $E_T(x)\in\mathcal X_T$ and $f_{0,T}(E_T(x))\le f_0(x)$ for every $x\in\mathcal X_0$; taking infima gives $L_T\le p_0^\star$. This pathwise argument covers adaptive or randomized policies and any finite stopping time, without optimizers, strong duality, nested relaxations, or training convergence. The stated terminal conditions define a convex feasible set by intersection; its convex objective completes the claim.
\end{proof}

Implementation of the admissibility checks is described in Appendix~\ref{sec_supp_implementation}.

\section{Local discrepancy bounds and mixed error propagation}\label{sec_supp_error}
The following local bounds and sensitivities compare original expressions with transformed representations in Theorem~\ref{thm_error}.

\subsection{Product and square discrepancies}
For a nondegenerate McCormick rectangle, normalize $u=\ell_u+\Delta_u s$, $v=\ell_v+\Delta_v t$, where $s,t\in[0,1]$, and define $h$ by
\[
w=\ell_u\ell_v+\ell_u\Delta_v t+\ell_v\Delta_u s+\Delta_u\Delta_v h.
\]
The envelope is $\max\{0,s+t-1\}\le h\le\min\{s,t\}$, and $w-uv=\Delta_u\Delta_v(h-st)$. If $s\le t$, then $\min\{s,t\}-st=s(1-t)\le s(1-s)\le1/4$; the case $t\le s$ is symmetric. If $s+t\le1$, then $st\le1/4$; if $s+t\ge1$, then $st-(s+t-1)=(1-s)(1-t)\le1/4$. Hence
\begin{equation}
|w-uv|\le\Delta_u\Delta_v/4.
\label{eq:mcc_error}
\end{equation}
A zero width forces the exact product, giving the same bound.

For $w$ representing $u^2$ with tangents at a set $\mathcal T$ containing the interval endpoints and with the secant upper bound,
\begin{equation}
w-u^2\le(u-\ell)(U-u),\qquad
u^2-w\le\min_{a\in\mathcal T}(u-a)^2.
\label{eq:square_error}
\end{equation}
Endpoint tangents bound both sides by $(U-\ell)^2/4$. Adding the midpoint tightens the second bound to $(U-\ell)^2/16$ without changing the secant bound. Nonnegativity and interval constraints can only tighten these finite-tangent bounds.

\subsection{SDP and spectral discrepancies}
The Schur complement of the unit leading entry gives $S=Z-zz^\top\succeq0$. For symmetric $Q$, spectral decomposition of $S=\sum_j s_j u_j u_j^\top$, $s_j\ge0$, yields
\begin{equation}
|\operatorname{Tr}(QS)|
\le\sum_j s_j|u_j^\top Qu_j|
\le\|Q\|_2\sum_j s_j
=\|Q\|_2\operatorname{Tr}(S).
\label{eq:sdp_error}
\end{equation}
For one-sided estimates write $Q=Q_+-Q_-$ with $Q_+,Q_-\succeq0$. Then
\[
-\operatorname{Tr}(Q_-S)\le\operatorname{Tr}(QZ)-z^\top Qz
\le\operatorname{Tr}(Q_+S).
\]
The condition $S=0$ suffices, but is not necessary, for a zero trace residual.

For QCR, Eq.~\eqref{eq:qcr_exact_error} and $(t-a)(b-t)\le(b-a)^2/4$ give
\begin{equation}
0\le q(z)-\widehat q(z)
\le\frac14\sum_{\lambda_k<0}|\lambda_k|(b_k-a_k)^2.
\label{eq:qcr_error_bound}
\end{equation}
The error in Eq.~\eqref{eq:qcr_exact_error} vanishes exactly when every nondegenerate negative-eigenvalue direction is at an endpoint.

\subsection{Fractional discrepancies and their coupling to product errors}
For a candidate ratio variable $\lambda$ with $|d|\ge\delta>0$,
\begin{equation}
\left|\lambda-\frac{p}{d}\right|
=\frac{|\lambda d-p|}{|d|}
\le\frac{|\lambda d-p|}{\delta}.
\label{eq:frac_residual}
\end{equation}
For $d=a^\top u+b$, let $w_j$ represent $\lambda u_j$ and $\widetilde p$ represent the numerator. With defining-relation residual $\epsilon_{\mathrm{eq}}=\sum_j a_j w_j+b\lambda-\widetilde p$, direct subtraction gives
\begin{equation}
\lambda d-p
=\epsilon_{\mathrm{eq}}+\sum_j a_j(\lambda u_j-w_j)+(\widetilde p-p).
\label{eq:frac_coupling}
\end{equation}
Consequently, if $|\widetilde p-p|\le E_p$ at the same parent values,
\begin{equation}
\left|\lambda-\frac{p}{d}\right|
\le\frac{|\epsilon_{\mathrm{eq}}|+\sum_j|a_j|\,|\lambda u_j-w_j|+E_p}{\delta}.
\label{eq:frac_combined_error}
\end{equation}
Retain $\epsilon_{\mathrm{eq}}$ if later rewrites make the equality inexact; division sensitivities below propagate denominator errors.

\subsection{Dependency graph, explicit sensitivities, and proof}
On a finite acyclic graph, use inputs $\bar x$, exact values $y_v^\star=\phi_v(y_{\operatorname{pa}(v)}^\star)$, and selected representations $\bar y_v$. Unchanged operations use represented parents. Record affine definitions and equality residuals, consistently reusing shared values and separating distinct representations. This analytical graph need not be the GNN graph.

Set $r_v=\eta_v=e_v=0$ at original inputs. Elsewhere, bound $r_v=\bar y_v-\phi_v(\bar y_{\operatorname{pa}(v)})$ directly or by the preceding local formulas. Certification requires parent boxes enclosing both exact and represented candidate values and keeping each operation well-defined; original-feasibility bounds alone may not cover an infeasible projection. The following constants ensure

\begin{equation}
|\phi_v(a)-\phi_v(b)|\le\sum_u L_{vu}|a_u-b_u|
\label{eq:coordinate_lipschitz}
\end{equation}
on the box:
\begin{itemize}
\item Affine $\phi(y)=c+\sum_j a_jy_j$: $L_j=|a_j|$.
\item Product $\phi(u,v)=uv$, $|u|\le M_u$, $|v|\le M_v$: $L_u=M_v$, $L_v=M_u$, by adding and subtracting $u'v$.
\item Quadratic $\phi(u)=u^\top Qu$, $Q=Q^\top$, $|u_j|\le M_j$: $L_j=2\sum_k|Q_{jk}|M_k$, using gradient $2Qu$ along the segment between parent tuples within the box.
\item Division $\phi(p,d)=p/d$, fixed-sign $d$ with $|d|\ge\delta$, $|p|\le M_p$: $L_p=1/\delta$, $L_d=M_p/\delta^2$, since $p/d-p'/d'=(p-p')/d+p'(d'-d)/(dd')$.
\end{itemize}
Count repeated inputs with multiplicity. Signed interval propagation may further tighten the constants.

\begin{proof}[Proof of Theorem~\ref{thm_error}]
At original inputs, $y_v^\star=\bar y_v$. For any subsequent node, Eq.~\eqref{eq:coordinate_lipschitz} gives
\begin{align}
|y_v^\star-\bar y_v|
&\le |\phi_v(y_{\operatorname{pa}(v)}^\star)-\phi_v(\bar y_{\operatorname{pa}(v)})|+|r_v|\notag\\
&\le\sum_u L_{vu}|y_u^\star-\bar y_u|+\eta_v
\le\sum_u L_{vu}e_u+\eta_v=e_v.
\label{eq:graph_induction}
\end{align}
One simultaneous topological induction covers all nodes and shared subexpressions. Unrolling gives
\begin{equation}
e_v=\sum_{w\preceq v}\eta_w
\sum_{P:w\leadsto v}\prod_{(u\to j)\in P}L_{ju},
\label{eq:path_error}
\end{equation}
with unit weight for a length-zero path and summed weights for repeated edges. The residuals and sensitivities depend on the construction, not action counts or fixed reward penalties.
\end{proof}

An illustrative mixed quadratic--fractional calculation is given in Appendix~\ref{sec_supp_mixed_example}.

\section{Recovery, exactness, and numerical certificates}\label{sec_supp_recovery}
\begin{proof}[Proof of Corollary~\ref{cor_recovery}]
For each original inequality, Theorem~\ref{thm_error} gives
\[
f_i(\bar x)\le\widetilde f_i(\bar z)+E_i\le0.
\]
Original-domain checks, including integrality, and affine equalities then establish $\bar x\in\mathcal X_0$; constraints without retained descendants require direct checks. Proposition~\ref{prop:composition} and certified $d_T\le L_T$ give $d_T\le p_0^\star\le f_0(\bar x)$. The objective discrepancy bound yields
\[
0\le f_0(\bar x)-p_0^\star
\le f_0(\bar x)-d_T
\le E_0+f_{0,T}(\bar z)-d_T.
\]
With finite terminal optimality and zero output discrepancies, terminal feasibility and the original-domain and equality checks give $\bar x\in\mathcal X_0$ and
\[
f_0(\bar x)=f_{0,T}(\bar z)=L_T\le p_0^\star\le f_0(\bar x).
\]
The equalities use zero objective discrepancy and terminal optimality; the inequalities use validity and feasibility. Hence $f_0(\bar x)=L_T=p_0^\star$ without $d_T=L_T$. Zero local residuals imply zero output errors via $\eta_v=0$, provided all active representations are checked: rank-one SDP blocks alone do not control McCormick, fractional, or spectral errors. Local errors may cancel, so zero residuals are sufficient but not necessary.
\end{proof}

\paragraph{Feasibility checks and recovered points}
Feasible recovery requires original-domain checks, including integrality, and constraint verification. Direct constraint evaluation provides an alternative when the sufficient margin test is inconclusive. A separately recovered feasible $x^{\mathrm{rec}}$ has gap at most $f_0(x^{\mathrm{rec}})-d_T$; applying the discrepancy bound to it requires bounds evaluated at that point.

\paragraph{Numerical certification}
A dual-feasible conic solution gives $d_T\le L_T$ by weak duality. Numerical certification requires verified feasibility, cone membership, residuals, and factorization errors. Empirical solver primal objectives are distinct from these certified dual bounds. The procedure is given in Appendix~\ref{sec_supp_certificate_checks}.

% Remaining supplementary content, integrated as Appendices E--J.
\widowpenalty=10000
\clubpenalty=10000
\displaywidowpenalty=10000
\section{Admissibility checks and certificate evaluation}\label{sec_supp_implementation}
The admissible action set in Eq.~\eqref{eq:admissible_action_set} combines target structure, variable domains, occurrence context, and shared-representation consistency. The location and operator masks in Eqs.~\eqref{eq:location_mask} and~\eqref{eq:operator_mask} restrict the policy to this set. With the conditions in Appendix~\ref{sec_supp_validity}, Proposition~\ref{prop:composition} guarantees validity along every finite admissible trajectory, independently of the learned policy parameters.

The constructions covered by this statement are those specified in the main text and Appendix~\ref{sec_supp_single_step}. In particular, QCR uses one chord per concave eigendirection, one-sided substitutions must be valid at every changed occurrence, and cuts must use the curvature of the complete signed expression. Auxiliary reuse must preserve recipe identity. Terminal convexity and numerical recovery certificates require the respective conditions in Appendices~\ref{sec_supp_trajectory_proof} and~\ref{sec_supp_recovery}.

Applying Proposition~\ref{prop:composition} to an executed trajectory requires joint checks of the sampling masks, occurrence-wise substitutions, signed curvature, and auxiliary reuse. Numerical masking must preserve zero probability for excluded actions after logit clipping. The experimental MOSEK interface returns primal objectives; certified numerical lower bounds require dual-feasible solutions and the checks in Appendix~\ref{sec_supp_certificate_checks}.

\subsection{Illustrative mixed-error calculation}\label{sec_supp_mixed_example}
Consider an objective $f=\alpha uv+z^\top Qz+p(x)/d(x)$ with representations $\widetilde f=\alpha w+\operatorname{Tr}(QZ)+\lambda$. If these factors use the original coordinates directly, $S=Z-zz^\top\succeq0$, and $|d(x)|\ge\delta$, then
\begin{equation}
|f-\widetilde f|
\le |\alpha|\,|uv-w|+\|Q\|_2\operatorname{Tr}(S)
+\frac{|\lambda d(x)-p(x)|}{\delta}.
\label{eq:mixed_example}
\end{equation}
If $p$ or the factors themselves have been lifted, Eqs.~\eqref{eq:frac_combined_error} and~\eqref{eq:error_recursion} propagate those additional discrepancies. A QCR term contributes Eq.~\eqref{eq:qcr_exact_error}. This example shows explicitly how different operators in one trajectory enter a common certificate without assuming one global relaxation family.

\subsection{Certificate evaluation procedure}\label{sec_supp_certificate_checks}
An instance-specific certificate uses the original formulation, the executed rewrites and shared recipes, terminal primal values, and a verified relaxation lower bound. Check the trajectory conditions of Proposition~\ref{prop:composition}, compute local residuals and candidate-valid enclosures, propagate the bounds, and verify original domains, equalities, and inequality margins. These checks establish the gap bound in Eq.~\eqref{eq:recovery_gap} for the selected trajectory. They constitute a post-processing certification procedure; the reported experiments evaluate solver objectives. Operator-specific reward weights guide training, while instance-specific residuals determine the certificate.

\section{Policy implementation and training}\label{sec_supp_policy_training}
\subsection{Policy architecture}\label{sec_supp_architecture}
The encoder produces node embeddings and type-wise pooled graph representations, as described in Eqs.~\eqref{eq:node_embedding} and~\eqref{eq:graph_pool}. Learned readouts map these representations to a common token dimension for self-attention and the factorized location--operator policy.

The implementation uses two heterogeneous message-passing layers of width 512. Node readout projects term representations from 512 to 256 dimensions; global readout projects the concatenated type-wise means from 1536 to 256 dimensions. The relation-specific updates use GATv2 for variable--term edges \citep{brody2022attentive}, graph convolution for nested-term edges, and mean-aggregating GraphSAGE for term--position edges \citep{hamilton2017inductive}. Equation~\eqref{eq:gnn_update} gives the mean-aggregation form of relation-wise message passing; these implementation choices specify the updates used for each relation in the experiments.

The factorized policy in the main paper is parameterized by a self-attention module and two masked heads. Its input is
\begin{equation}
H_t =
[z_{\ell_0}, z_{\ell_1}, \cdots,z_{\ell_{n_t}}],
\quad
\mathrm{with} \ \ell_0=\ell_{\mathrm{glob}},\
z_{\ell_0}=z_{\widetilde G_t}.
\label{eq:input}
\end{equation}
Here, $z_{\widetilde{G}}$ and $z_{\ell_i}$ denote the 256-dimensional readouts of the graph representation in Eq.~\eqref{eq:graph_pool} and the candidate-location embeddings in Eq.~\eqref{eq:node_embedding}, respectively.

Multi-head self-attention \citep{vaswani2017attention} contextualizes the input tokens:
\begin{equation}
\widetilde H_t =
\operatorname{LN}\big(H_t+\operatorname{MHA}(H_t,H_t,H_t)\big),
\label{eq:policy_attention}
\end{equation}
where $\operatorname{MHA}$ denotes multi-head self-attention with queries, keys, and values all given by $H_t$; $\operatorname{LN}$ denotes layer normalization.

A learnable location head $\psi_{\mathrm{loc}}$ maps the contextualized tokens to location logits. The location policy is the masked softmax
\begin{equation}
\pi_\theta^{(\ell)}(\cdot\mid\mathcal{F}_t) = \mathrm{softmax}_{i\in\mathcal{I}_t}\big(\psi_{\mathrm{loc}}(\widetilde H_t)\big).
\label{eq:location_policy}
\end{equation}
where $\mathcal{I}_t$ is the admissible token-index set. With $H_t$ (Eq. \eqref{eq:input}), this set is
\begin{equation}
\mathcal{I}_t =
\begin{cases}
\{i:1\le i\le n_t,\ \exists o\in\mathcal{O},\;(\ell_i,o)\in\mathcal{A}_t\},
& \Lambda_t\neq\emptyset,\\
\{0:\exists o\in\mathcal{O},\;(\ell_{\mathrm{glob}},o)\in\mathcal{A}_t\},
& \Lambda_t=\emptyset.
\end{cases}
\label{eq:location_mask}
\end{equation}
Thus, while nonconvex term locations remain, the global token is masked out and the policy selects among admissible target terms. After convexification, the global token is exposed for formulation-level actions such as bound tightening or cut generation.

Given the selected location index $i_t$, the operator policy uses a second masked head
\begin{equation}
\pi_\theta^{(o)}(o_t=o\mid i_t,\mathcal{F}_t) = \mathrm{softmax}_{o\in\mathcal{M}_t^{(o)}(i_t)}
\left(\psi_{\mathrm{op}}\big([\widetilde H_t[i_t];q_{\ell_t}]\big)\right)_o,
\label{eq:operator_policy}
\end{equation}
where $\ell_t=\ell_{i_t}$ with the convention $\ell_0=\ell_{\mathrm{glob}}$, and $q_{\ell_t}$ encodes operator-relevant attributes of the selected location. The mask $\mathcal M_t^{(o)}(i_t)$ is induced by the admissible rewrite set in Eq.~\eqref{eq:operator_mask}, assigning zero probability to excluded operators and enforcing domain, curvature, and occurrence-context conditions.

One eight-head self-attention layer contextualizes the resulting global and candidate-term tokens. The location head is a $256$--$512$--$1$ multilayer perceptron. The operator head concatenates the selected token with four binary attributes for continuous, integer, binary, and fractional structure, giving a $260$--$512$--$7$ multilayer perceptron. The value head uses a $256$--$256$--$1$ architecture with a Tanh activation. All modules are trained jointly. Table~\ref{tab:policy_architecture} summarizes the encoder and policy settings.

\begin{table}[t]
\centering
\caption{Encoder and policy architecture settings.}
\label{tab:policy_architecture}
\small
\renewcommand{\arraystretch}{1.12}
\begin{tabular*}{\textwidth}{@{\extracolsep{\fill}}ll@{}}
\toprule
Component & Configuration \\
\midrule
Message-passing layers & 2 \\
Graph hidden dimension & 512 \\
Unprojected pooled dimension & $3\times512=1536$ \\
Node / global readout & $512\to256$ / $1536\to256$ \\
Policy-token dimension & 256 \\
Graph pooling & Type-wise mean pooling \\
Graph dropout & 0.1 \\
Self-attention layers / heads & 1 / 8 \\
Location head & $256$--$512$--$1$ \\
Operator head & $260$--$512$--$7$ \\
Value head & $256$--$256$--$1$ \\
\bottomrule
\end{tabular*}
\end{table}

\subsection{Reward and policy optimization}\label{sec_supp_reward}
Starting from the original formulation $\mathcal{F}_0$, the policy generates
\begin{equation}
\zeta=(\mathcal{F}_0,a_0,\mathcal{F}_1,a_1,\cdots,a_{T-1},\mathcal{F}_T).
\label{eq:trajectory}
\end{equation}
The episode terminates when the horizon is reached, or when the formulation is convex and the budget for post-convex global actions is exhausted.

We train the policy with shaped step rewards and terminal solver feedback. The step reward is
\begin{equation}
\begin{aligned}
\bar r_t={}&
\alpha^\top
\begin{bmatrix}
r_t^{\mathrm{prog}}\\
r_t^{\mathrm{struct}}\\
r_t^{\mathrm{dense}}
\end{bmatrix}
-\beta\,\widetilde{\mathrm{Cost}}_t-\xi_t,\\
r_t={}&r_{\max}\tanh\!\left(\frac{\bar r_t}{r_{\max}}\right),
\end{aligned}
\label{eq:step_reward}
\end{equation}
where $r_t^{\mathrm{prog}}$ rewards convexification progress, $r_t^{\mathrm{struct}}$ rewards structural simplification, and $r_t^{\mathrm{dense}}$ rewards reduction of nonconvex terms. $\widetilde{\mathrm{Cost}}_t$ penalizes estimated formulation growth, $\xi_t$ collects penalties for approximation error and ineffective rewrites, and $r_{\max}$ bounds the step reward.

If the terminal formulation is convex, it is solved numerically and a solver bonus is added:
\begin{equation}
r_T^{\mathrm{solver}}
=
\eta
\tanh\!\left(
\frac{\mathrm{LB}(\mathcal{F}_T)-\mathrm{LB}_{\mathrm{root}}}
{s\max\!\big(|\mathrm{LB}_{\mathrm{root}}|,\delta\big)}
\right),
\label{eq:solver_bonus}
\end{equation}
where $\mathrm{LB}_{\mathrm{root}}$ is the cached root lower bound of the original instance, while $\eta$, $s$, and $\delta$ control reward scale and numerical stability.

The progress components above are computed from
the symbolic change produced by each rewrite. The direct-progress term
records whether the rewritten term enters a supported convexity class
and includes a bonus when the formulation first becomes convex. The
structural term records whether newly exposed terms are supported and
otherwise aggregates reductions in fractional terms, expression depth,
nonsmooth terms, and multiplicative terms. The dense term is the
relative reduction in nonconvex terms.

The formulation-growth term is
\begin{equation}
\widetilde{\mathrm{Cost}}_t=c_{o_t}+\Delta C_t,
\end{equation}
where $c_{o_t}$ is the nominal cost of the selected operator. Its
state-dependent component is
\begin{equation}
\begin{aligned}
\Delta C_t={}&
w_{\mathrm{var}}\Delta n_{\mathrm{var}}
+w_{\mathrm{mat}}\Delta n_{\mathrm{mat}}
+w_{\mathrm{aux}}\Delta n_{\mathrm{aux}}
+w_{\mathrm{term}}\Delta n_{\mathrm{term}}\\
&+w_{\mathrm{cons}}\Delta n_{\mathrm{cons}}
+w_{\mathrm{PSD},2}\Delta d_{\mathrm{PSD}}^{(2)}
+w_{\mathrm{PSD},3}\Delta d_{\mathrm{PSD}}^{(3)}.
\end{aligned}
\end{equation}
Each $\Delta$ is the nonnegative increase from $\mathcal{F}_t$ to
$\mathcal{F}_{t+1}$. The subscripts $\mathrm{var}$, $\mathrm{mat}$,
$\mathrm{aux}$, $\mathrm{term}$, and $\mathrm{cons}$ denote scalar
variables, matrix-variable entries, generated auxiliary variables,
symbolic terms, and constraints. Auxiliary variables are counted again
by $\Delta n_{\mathrm{aux}}$ to penalize lifting more strongly. If the
PSD cones have dimensions $d_j$, then
$\Delta d_{\mathrm{PSD}}^{(p)}$ is the increase in
$\sum_j d_j^p$ for $p\in\{2,3\}$, approximating storage and
factorization costs.

The total penalty is
\begin{equation}
\xi_t=\beta_{\mathrm{error}}e_{o_t}+\xi_t^{\mathrm{ineff}},
\end{equation}
where $e_{o_t}$ is the operator-specific approximation-error scale and
$\xi_t^{\mathrm{ineff}}$ is nonzero only for an ineffective rewrite.
Post-convex global actions receive a positive step signal only when they
tighten variable bounds.

The terminal solver bonus above is evaluated on a
sample of completed convex rollouts. Each training-time solve has a
fixed time limit, and failed solves receive a fixed penalty.

We use PPO with standardized returns and advantages. The coefficients prioritize direct convexification over intermediate structural progress, while penalizing excessive formulation growth and weak approximations. Their values place heterogeneous reward components on comparable scales, bound outliers, and stabilize training.

We optimize the policy with proximal policy optimization \citep{schulman2017proximal}. For action $a_t=(\ell_t,o_t)$, its probability is given by the factorized location and operator heads in Eq.~\eqref{eq:factor_policy}; the log probability is the sum of the two log factors. The training loss is
\begin{equation}
\mathcal{L}(\theta)
=
-
\mathbb{E}_t
\left[
\min\left(
\rho_t(\theta)\hat{A}_t,\;
\operatorname{Clip}\big(\rho_t(\theta),1-\epsilon,1+\epsilon\big)\hat{A}_t
\right)
\right]
+
\lambda_v L^{\mathrm{value}}
-
\lambda_e \mathcal{H}(\pi_\theta),
\label{eq:ppo_loss}
\end{equation}
where
\begin{equation}
\rho_t(\theta)
=
\frac{\pi_\theta(a_t\mid\mathcal{F}_t)}
{\pi_{\theta_{\mathrm{old}}}(a_t\mid\mathcal{F}_t)} .
\label{eq:ratio}
\end{equation}
Here, $\mathbb{E}_t$ averages over sampled rollout steps; $\rho_t(\theta)$ is the policy probability ratio; $\theta_{\mathrm{old}}$ denotes the fixed
policy parameters used to collect the rollout; $\hat{A}_t$ is the advantage estimate from the learned value function; $\epsilon$ is the clipping range;
$L^{\mathrm{value}}$ is the value-function loss; $\mathcal{H}(\pi_\theta)$ is the policy entropy; $\lambda_v,\lambda_e$ weight the value and entropy terms. All neural components, including the graph encoder, self-attention module, and policy heads, are updated jointly from interactions with instances drawn from the training distribution. 

\subsection{Training configuration}\label{sec_supp_training}
The architecture and reward details are given in Appendices~\ref{sec_supp_architecture} and~\ref{sec_supp_reward}. Table~\ref{tab:policy_optimization} lists their numerical settings. Coefficients were set from training rollouts, fixed before evaluation, and never tuned on held-out instances. Operator costs and errors follow the order integrality relaxation, fractional removal, McCormick, SDP, QCR, bound tightening, and cut generation.
\begin{table}[htbp]
\centering
\caption{Policy optimization and implementation settings.}\label{tab:policy_optimization}
\footnotesize
\renewcommand{\arraystretch}{0.95}
\begin{tabular*}{\textwidth}{@{\extracolsep{\fill}}>{\raggedright\arraybackslash}p{0.27\textwidth}>{\raggedright\arraybackslash}p{0.18\textwidth}>{\raggedright\arraybackslash}p{0.27\textwidth}>{\raggedright\arraybackslash}p{0.18\textwidth}@{}}
\toprule
Setting & Value & Setting & Value \\
\midrule
Optimizer & Adam & Learning rate / epsilon & $10^{-4}$ / $10^{-8}$ \\
Discount factor $\gamma$ & 0.99 & PPO clip ratio $\epsilon$ & 0.15 \\
Value loss / clipping & 0.5 / 0.2 & Entropy coefficient & $0.03$, then $0.02\to0.005$ \\
PPO epochs / minibatch & 1 / 128 & Problems per update & 32 \\
Rollout steps / problem & 16 & Maximum gradient norm & 0.5 \\
Training / post-convex horizon & 20 / 4 & Progress signal & $1$ / $-0.1$ \\
Structural-reduction scales & \makecell[l]{$(0.20,0.15,$\\$0.25,0.20)$} & Step scales $\alpha$ & $(2,1.5,0.35)$ \\
Penalty scales & $(0.1,0.5,2.5)$ & Linear cost scales & \makecell[l]{$(0.020,0.012,0.050,$\\$0.010,0.012)$} \\
PSD cost scales & $(0.006,0.0003)$ & Ineffective / convexity bonus & $0.45$ / $2.0$ \\
Terminal scales $(\eta,s,\delta)$ & $(6,0.20,1)$ & Sampling / solve / failure & $20\%$ / $10$ s / $1.0$ \\
Operator costs $c_o$ & \makecell[l]{$(0.02,0.03,0.18,$\\$0.30,0.20,0.18,$\\$0.22)$} & Operator errors $e_o$ & \makecell[l]{$(0.25,0.10,0.20,$\\$0.05,0.08,0.00,$\\$0.02)$} \\
\bottomrule
\end{tabular*}
\end{table}

\subsection{Training and fine-tuning protocol}\label{sec_supp_training_protocol}
The standard-QCQP policy is trained in two stages. The first stage uses 200 updates with shaped step rewards. The second initializes from this checkpoint and uses 200 solver-guided updates. The 241 held-out QCQPs provide no trajectories, rewards, solver feedback, or parameter updates.

The fractional policy initializes from the standard-QCQP checkpoint. It is fine-tuned for 200 updates on a separate fractional training split using the same architecture and PPO settings. The 15 reported fractional instances remain excluded from fine-tuning. Application experiments use the standard-QCQP checkpoint directly. All parameters remain fixed, and no application instance is used for tuning.

\section{Benchmark instance construction}\label{sec_supp_instances}
\subsection{Standard QCQP benchmarks}\label{sec_supp_qcqp}
The benchmark set consists of synthetic standard QCQPs and fractional nonconvex problems with quadratic structure. They are constructed to expose different sources of nonconvexity. Each instance retains its variables, symbolic terms, bounds, and formulation locations. The standard-QCQP pool is split into training and held-out subsets in an 80/20 ratio. The main evaluation uses 241 held-out standard QCQPs. The fractional study uses a separate training split for fine-tuning and 15 held-out fractional instances for evaluation.

The held-out set covers seven dominant structural families. Mixed-structure QCQPs contain several nonconvex components in the same formulation. Dense quadratic blocks contain dense indefinite or concave quadratic terms. Diagonal or separable quadratics contain mostly coordinate-wise quadratic structure. Bilinear continuous instances contain products between continuous variables. Binary-coupled instances contain binary or integer variables coupled with continuous variables. Nonconvex-constraint instances contain quadratic nonconvexity in constraints. Concave-objective instances contain nonconvex curvature primarily in the objective. These labels are used only for analysis and plotting; the policy receives the symbolic formulation and graph representation, not the family label.

\subsection{Fractional benchmark structure}\label{sec_supp_fractional}
The held-out fractional generators combine an indefinite quadratic ratio with additional quadratic, bilinear, or binary--continuous terms, or place the ratio in a constraint. The variable-dependent denominator has the form
\begin{equation}
d(z)=6+\frac{1}{n}\sum_{j=1}^{n}z_j,
\end{equation}
with fixed coefficients and bounded continuous variables. Its interval bounds exclude zero, as required by the fractional-removal action in the main paper. Ratios are retained as single symbolic terms rather than distributed into separate monomial fractions. This keeps the denominator structure explicit for the action system. The generator uses vector dimensions of 5--7, or 4--5 for the binary--continuous family; indefinite quadratic blocks and dense bilinear couplings create the nonconvexity.

\section{Experimental settings}\label{sec_supp_settings}
\subsection{Baseline construction rules}\label{sec_supp_baselines}
The comparison methods and their selection rationale are given in Section~\ref{sec_experimental_settings} of the main paper. Their construction rules are as follows.

The same symbolic engine constructs all handcrafted endpoints. McCormick first relaxes integrality and safe fractional structure, then envelopes supported bilinear, square, inner-product, and affine-product terms. SDP uses the same preprocessing, lifts recognizable quadratic forms, and adds positive-semidefinite, symmetry, bound-derived, RLT, and McCormick constraints when supported; McCormick is the fallback for unsuitable terms. Random draws one legal rollout per instance from a seed determined by the experiment seed, rollout index, and instance name. The structure-based heuristic uses a fixed term order, applies McCormick to approximately separable terms, and applies SDP to coupled terms. For coefficient matrix $Q$ with $m$ active variables, it sets
\begin{equation}
\operatorname{useSDP}(Q)
= (m\ge k_{\min})\wedge(\rho(Q)\ge\tau_{\mathrm{density}}),
\end{equation}
and chooses SDP exactly when this condition holds, where
\begin{equation}
\rho(Q)
=
\frac{\mathrm{nnz}(\mathrm{offdiag}(Q))}{\max\{1,m(m-1)\}}.
\end{equation}
The density lies in $[0,1]$; we set $k_{\min}=4$ and $\tau_{\mathrm{density}}=0.75$.

\subsection{Evaluation metrics}\label{sec_supp_metrics}
Bound quality is measured by the normalized lower-bound difference between a learned value $v_{\mathrm{RL}}$ and a reference value $v_{\mathrm{ref}}$:
\begin{equation}
\Delta_{\mathrm{LB}}=100\,
\frac{v_{\mathrm{RL}}-v_{\mathrm{ref}}}
{\max\{|v_{\mathrm{RL}}|,|v_{\mathrm{ref}}|,10^{-9}\}}.
\label{eq:bound_metric}
\end{equation}
Positive values indicate tighter lower bounds. Pairwise tighter/tie/weaker outcomes use a tolerance of $10^{-6}$. Formulation growth is measured by relative changes in variable, constraint, and nonzero counts. The bound--cost comparison therefore measures cost through formulation size. Equation~\eqref{eq:bound_metric} reports normalized lower-bound differences relative to each reference relaxation. For application instances, surrogate objectives are compared with the best objective values returned by the reference solvers; beamforming uses upper-bounding relaxations and localization uses lower-bounding relaxations.

\subsection{Rollout and solver protocol}\label{sec_supp_protocol}
Evaluation greedily decodes both policy heads, allows at most 50 actions and six post-convex global actions, and stops after two consecutive ineffective rewrites. Convex terminal surrogates are exported to MOSEK Fusion as continuous affine, rotated-conic, or semidefinite models. Terminal surrogates receive 20 seconds within a 60-second comparison budget; Gurobi and SCIP root relaxations receive 60 seconds. Application surrogates and nonconvex reference solves receive 30 and 60 seconds, respectively. All 241 held-out QCQP instances have valid SCIP root bounds; the SCIP mean improvement and tighter/tie/weaker percentages use this full set. Experiments use Python 3.10.18, PyTorch 2.4.1, CUDA 12.1, and NVIDIA V100 GPUs.

\subsection{Structure-dependent bounds and formulation growth}\label{sec_supp_structure_details}
Fig.~\ref{fig:family_effect_size} shows that this adaptation delivers its largest gains where relaxation design has the highest leverage. The largest root-bound gains appear on binary-coupled quadratics, dense quadratic terms, mixed-structure QCQPs, and nonconvex quadratic constraints. These are high-leverage regimes. The Gurobi root relaxation is often loose, and stronger lifting or reformulation actions can materially tighten the lower bound. Bilinear continuous couplings show a modest positive effect, suggesting that the benefit is present but less systematic in weaker coupling regimes. Diagonal or separable quadratics show a negative mean improvement. These simpler structures offer less scope for bound improvement through additional formulation growth. The results indicate that the learned policy is most useful on heterogeneous and strongly coupled structures, where fixed lightweight relaxations are often insufficient.
\begin{figure}
    \centering
    \includegraphics[width=0.6\linewidth]{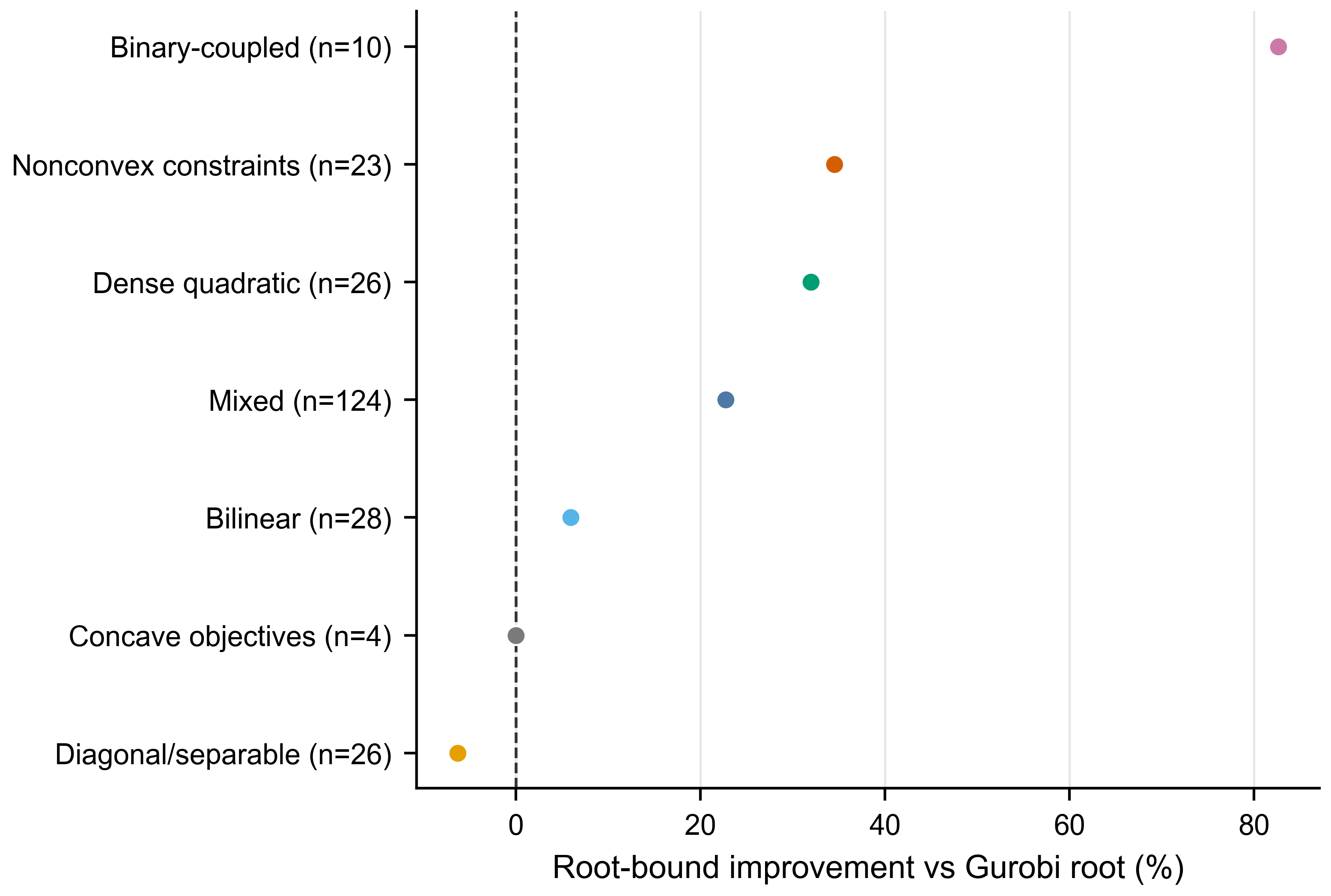}
    \caption{Structure-dependent root-bound improvement of learned trajectories. Points denote the mean improvement over the Gurobi root relaxation within each QCQP structural family; error bars are computed over random seeds.}
    \label{fig:family_effect_size}
\end{figure}

Fig.~\ref{fig:family_cost_tradeoff} shows how this structural adaptation translates into bound--cost trade-offs. Using 50\% root-bound improvement as a high-return threshold, 39 plotted instances exceed this level. These high-return cases are concentrated in mixed-structure QCQPs, nonconvex quadratic constraints, binary-coupled quadratics, and dense quadratic terms. No bilinear continuous or diagonal/separable instance reaches this threshold. The opposite pattern appears in the high-growth region. Among the 21 instances with at least \(10^3\) added nonzeros, 16 obtain no more than 5\% improvement. Thus, formulation growth is not a monotone proxy for relaxation strength. The useful trade-offs appear when stronger relaxation mechanisms are applied to coupled or constraint-level nonconvexity, where these mechanisms have high leverage. Together with the action-composition results, this shows that the policy adapts both which relaxation mechanisms to use and where added formulation size translates into bound improvement.
\begin{figure}
    \centering
    \includegraphics[width=0.6\linewidth]{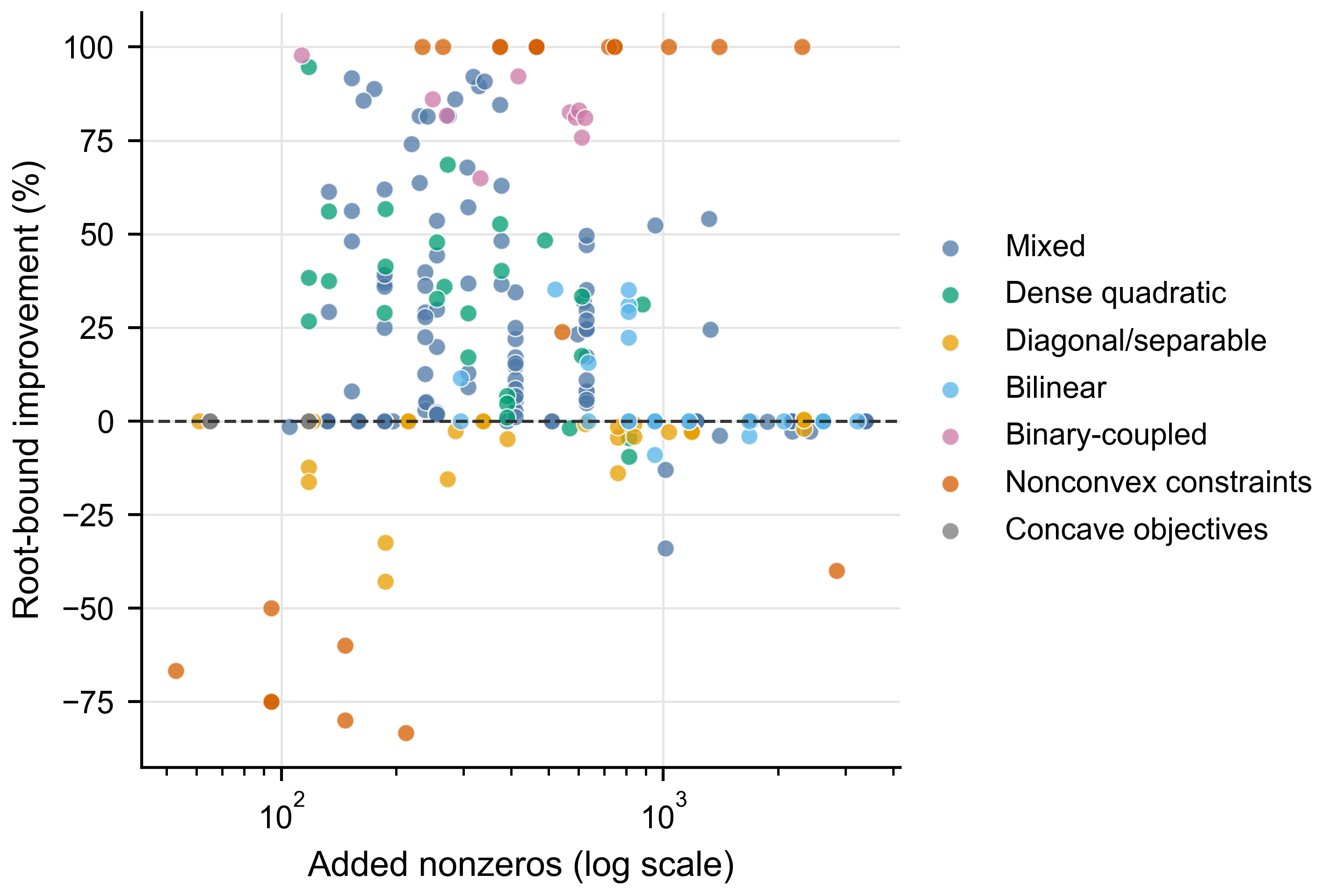}
    \caption{Bound--cost trade-off of learned trajectories. Each point represents a held-out instance averaged over converged random seeds. The horizontal axis reports added nonzeros as a proxy for formulation growth on a logarithmic scale; the vertical axis reports root-bound improvement over the Gurobi root relaxation. Colors indicate QCQP structural families.}
    \label{fig:family_cost_tradeoff}
\end{figure}

\section{Detailed term-level trajectory inspection}\label{sec_supp_trajectories}
This section reports the full algebraic rewrites for the three trajectories summarized in the main text. It also analyzes case \texttt{H\_MIXED\_1\_n4\_929}, which exposes term-level selectivity when the solver root relaxation is already tight.

\subsection{Placement and sequencing in representative trajectories}\label{sec_supp_trajectory_analysis}
Table~\ref{tab:term_level_trajectories} summarizes the trajectories. \texttt{H\_MIXED\_5\_n3\_4\_615} exposes why relaxation design should operate at the level of individual nonconvex terms. The problem contains two quadratic terms,
\begin{equation}
\begin{aligned}
P^{(0)}:\quad
\min_{x,y}\quad &
x^\top Q_x x + y^\top Q_y y \\
\mathrm{s.t.}\quad &
\|x\|_2^2+\|y\|_2^2\le 16.3616,
\end{aligned}
\label{eq:hmixed5_p0}
\end{equation}
where $Q_x=\mathrm{diag}(-0.396,0.800,-0.943)$ and $Q_y$ is dense. A global template cannot distinguish these two sources of nonconvexity. The McCormick endpoint leaves the bound unchanged at $-29.73$, whereas full SDP lifts the whole quadratic structure. The learned trajectory makes a finer decision. It applies McCormick relaxation to the two concave diagonal quadratic terms,
\begin{equation}
x^\top Q_x x
\;\Longrightarrow\;
-0.396\,w_{x_0x_0}
+0.800\,w_{x_1x_1}
-0.943\,w_{x_2x_2},
\qquad
(x,w_x)\in\mathcal{M}_x,
\label{eq:hmixed5_mcc}
\end{equation}
and reserves SDP lifting for the dense term,
\begin{equation}
y^\top Q_y y
\;\Longrightarrow\;
\mathrm{Tr}(Q_yZ_y),
\qquad
\begin{bmatrix}
1 & y^\top\\
y & Z_y
\end{bmatrix}\succeq 0.
\label{eq:hmixed5_sdp}
\end{equation}
The resulting surrogate reaches a lower bound of $-11.603$, close to the SDP bound $-11.30$, while avoiding a uniform SDP treatment of both terms. This example shows more than operator mixing. It shows localized relaxation strength. The policy uses a cheap envelope where the structure is separable and spends matrix lifting where dense coupling controls the relaxation gap. Relaxation strength is therefore allocated at the term level rather than applied uniformly.
\begin{table}[t]
\centering
\caption{Relaxation trajectories produced by the policy. The trajectories are reported as decision-level summary of the target term, relaxation operator, and rewrite role. Full algebraic rewrites are provided in Appendix~\ref{sec_supp_trajectories}.}
\label{tab:term_level_trajectories}
\footnotesize
\renewcommand{\arraystretch}{1.18}
\setlength{\tabcolsep}{3pt}
\begin{tabular*}{\linewidth}{@{\extracolsep{\fill}}>{\raggedright\arraybackslash}p{2.2cm}>{\raggedright\arraybackslash}p{0.75cm}>{\raggedright\arraybackslash}p{3.0cm}>{\raggedright\arraybackslash}p{2.35cm}>{\raggedright\arraybackslash}p{5.6cm}@{}}
\toprule
Instance & Step & Target term or structure & Operator & Rewrite and role \\
\midrule
\multirow{2}{=}{H\_MIXED\_5\\\_n3\_4\_615}
& 1
& diagonal term $x^\top Q_x x$
& McCormick
& $x^\top Q_xx \mapsto \sum_i q_i w_{x_ix_i}$; envelope relaxation of the diagonal quadratic term \\
& 2
& dense term $y^\top Q_y y$
& SDP
& $y^\top Q_yy \mapsto \mathrm{Tr}(Q_yZ_y)$; matrix lifting for the dense quadratic term \\
\midrule
\multirow{3}{=}{E\_EXP\_P3\\\_n4\_945}
& 1
& binary term $b\in\{0,1\}^4$
& RelaxIntegrality
& $b\in\{0,1\}^4 \mapsto b\in[0,1]^4$; domain relaxation before lifting \\
& 2
& bilinear coupling $2b^\top Ax$
& McCormick
& $2b^\top Ax \mapsto \langle C,w_{bx}\rangle$; envelope relaxation of binary--continuous coupling \\
& 3
& concave quadratic $-0.4375\|x\|_2^2$
& SDP
& $-0.4375\|x\|_2^2 \mapsto -0.4375\,\mathrm{Tr}(Z_x)$; lifting of the remaining concave quadratic term \\
\midrule
\multirow{4}{=}{H\_MIXED\_1\\\_n4\_929}
& 1
& diagonal bilinear term $2x^\top Dy$
& McCormick
& $2x^\top Dy \mapsto \sum_i 2d_i w_{x_iy_i}$; envelope relaxation of cross-variable coupling \\
& 2
& uniform bilinear term $0.34235\,x^\top y$
& McCormick
& $0.34235\,x^\top y \mapsto 0.34235\sum_i w_{x_iy_i}$; envelope relaxation of cross terms \\
& 3
& diagonal quadratic $y^\top Qy$
& McCormick
& $y^\top Qy \mapsto \sum_i q_i w_{y_iy_i}$; envelope relaxation of the separable quadratic term \\
& 4
& remaining quadratic term $x^\top Qx$
& SDP
& $x^\top Qx \mapsto \mathrm{Tr}(QZ_x)$; matrix lifting for the remaining quadratic term \\
\bottomrule
\end{tabular*}
\end{table}

\texttt{E\_EXP\_P3\_n4\_945} illustrates sequencing in relaxation design. Unlike the previous example's term-wise placement, this instance contains multiple interacting sources of nonconvexity: a binary decision term, a binary--continuous coupling, and a concave quadratic term:
\begin{equation}
\begin{aligned}
P^{(0)}:\quad
\min_{x,b}\quad &
\|b\|_2^2
-0.4375\|x\|_2^2
+2b^\top A x
+x^\top Qx \\
\mathrm{s.t.}\quad &
b\in\{0,1\}^4,\quad
x_i \le 2.23395\,b_i,\\
&
\|b\|_2^2 \le 3.2277,\quad
\|x\|_2^2 \le 18.1793,\quad
\mathbf{1}^\top b = 1.
\end{aligned}
\end{equation}

The learned trajectory does not apply relaxation operators simultaneously. Instead, it constructs a decomposition sequence of the problem structure. It first resolves the discrete structure
\begin{equation}
b\in\{0,1\}^4 \;\Longrightarrow\; b\in[0,1]^4,
\end{equation}
then convexifies the binary--continuous coupling via McCormick envelopes
\begin{equation}
2b^\top A x \;\Longrightarrow\; \langle C,w_{bx}\rangle,
\qquad (b,x,w_{bx})\in\mathcal{M}_{bx},
\end{equation}
and finally applies SDP lifting to the remaining concave quadratic component
\begin{equation}
-0.4375\|x\|_2^2 \;\Longrightarrow\; -0.4375\,\mathrm{Tr}(Z_x),
\qquad
\begin{bmatrix}
1 & x^\top\\
x & Z_x
\end{bmatrix}\succeq 0.
\end{equation}

This sequencing is not dictated by a predefined relaxation template. It reflects a learned structural decomposition strategy: discrete constraints are relaxed first to expose a continuous domain, coupling terms are then convexified locally, and SDP lifting is reserved for the residual concave curvature. The resulting surrogate achieves a lower bound of $-34.28$, matching the SDP relaxation at the reported precision and improving over the McCormick endpoint of $-35.06$. This example demonstrates a stronger form of compositionality than operator selection alone. The policy does not only decide which relaxation mechanisms to use, but also how to organize them into a sequential program over progressively simplified formulations. This yields a relaxation design that is inherently ordered rather than a static collection of operators.

These two inspections make the learned trajectory explicit. It is not only a list of relaxation operators. It specifies where each mechanism is placed and how the resulting rewrites are sequenced to form the convex surrogate.

\subsection{Trajectory for \texttt{H\_MIXED\_5\_n3\_4\_615}}\label{sec_supp_trajectory1}
The original objective contains a diagonal quadratic block \(x^\top Q_xx\) and a dense quadratic block \(y^\top Q_yy\). The policy first applies McCormick relaxation to the diagonal block,
\begin{equation}
x^\top Q_xx
\Longrightarrow
-0.396\,w_{x_0x_0}+0.800\,w_{x_1x_1}-0.943\,w_{x_2x_2},
\qquad (x,w_x)\in\mathcal{M}_x,
\end{equation}
where \(\mathcal{M}_x\) denotes the McCormick envelope constraints for the corresponding square terms. The policy then applies SDP lifting to the dense block,
\begin{equation}
y^\top Q_yy
\Longrightarrow
\mathrm{Tr}(Q_yZ_y),
\qquad
\begin{bmatrix}
1 & y^\top\\
y & Z_y
\end{bmatrix}\succeq 0.
\end{equation}
The final surrogate combines a lightweight envelope for the separable diagonal term with a matrix lift for the dense term.

\subsection{Trajectory for \texttt{E\_EXP\_P3\_n4\_945}}\label{sec_supp_trajectory2}
This instance contains binary variables \(b\), binary--continuous coupling, and concave quadratic curvature. The policy first relaxes the binary domain,
\begin{equation}
b\in\{0,1\}^4 \Longrightarrow b\in[0,1]^4.
\end{equation}
It then applies McCormick envelopes to the binary--continuous coupling,
\begin{equation}
2b^\top Ax \Longrightarrow \langle C,w_{bx}\rangle,
\qquad (b,x,w_{bx})\in\mathcal{M}_{bx}.
\end{equation}
Here $C=2A$, $(w_{bx})_{ij}$ is the auxiliary variable representing $b_i x_j$, and $\langle C,w_{bx}\rangle=\sum_{i,j}C_{ij}(w_{bx})_{ij}$ is the Frobenius inner product; $\mathcal{M}_{bx}$ denotes the corresponding McCormick envelope constraints.
Finally, it lifts the remaining concave quadratic component,
\begin{equation}
-0.4375\|x\|_2^2
\Longrightarrow
-0.4375\,\mathrm{Tr}(Z_x),
\qquad
\begin{bmatrix}
1 & x^\top\\
x & Z_x
\end{bmatrix}\succeq 0.
\end{equation}
The ordering separates domain relaxation, local product convexification, and stronger matrix lifting.

\subsection{Trajectory for \texttt{H\_MIXED\_1\_n4\_929}}\label{sec_supp_trajectory3}
The objective contains two cross-variable bilinear structures and two diagonal quadratic blocks. The first two steps apply McCormick envelopes to the bilinear terms,
\begin{equation}
2x^\top Dy
\Longrightarrow
-2.572\,w_{x_0y_0}+2.834\,w_{x_1y_1}+5.174\,w_{x_2y_2}+0.644\,w_{x_3y_3},
\end{equation}
and
\begin{equation}
0.34235\,x^\top y
\Longrightarrow
0.34235\sum_{i=0}^{3}w_{x_iy_i}.
\end{equation}
The third step applies McCormick relaxation to the separable \(y\)-quadratic block,
\begin{equation}
y^\top Qy
\Longrightarrow
1.286\,w_{y_0y_0}-1.417\,w_{y_1y_1}-2.587\,w_{y_2y_2}-0.322\,w_{y_3y_3}.
\end{equation}
The final step applies SDP lifting to the remaining \(x\)-quadratic block,
\begin{equation}
x^\top Qx
\Longrightarrow
\mathrm{Tr}(QZ_x),
\qquad
\begin{bmatrix}
1 & x^\top\\
x & Z_x
\end{bmatrix}\succeq 0.
\end{equation}

\paragraph{Analysis}
For \texttt{H\_MIXED\_1\_n4\_929}, the learned trajectory improves the McCormick lower bound from $-85.65$ to $-84.82$, with solver-root and SDP reference values of $-84.40$ (Table~\ref{tab:representative_trajectories}). The instance illustrates how the policy combines relaxation mechanisms when the reference bounds are already tight.

The policy applies McCormick envelopes to the cross-variable products and separable quadratic components, then reserves SDP lifting for the remaining quadratic block. This sequence separates local product structure from residual curvature and assigns each component a corresponding relaxation mechanism. The trajectory thus provides a concrete example of term-level placement and adaptive sequencing.

\section{Application formulations and instance construction}\label{sec_supp_applications}
Application results are reported in the main paper. Beamforming introduces dense quadratic coupling through signal-quality requirements and transmit-power limits, while localization introduces geometric coupling through noisy distance constraints. The formulations and generation settings are specified here; their evaluation uses the frozen policy described in Appendix~\ref{sec_supp_training_protocol} and the solver budgets in Appendix~\ref{sec_supp_protocol}.
\subsection{Multicast beamforming}\label{sec_supp_beamforming}
The beamforming application uses a real-valued max-min multicast beamforming QCQP. For each instance, a beamformer \(w\in\mathbb{R}^{n}\) and scalar margin \(t\) are optimized under quadratic signal-quality constraints,
\begin{equation}
\begin{aligned}
\max_{w,t}\quad & t\\
\mathrm{s.t.}\quad & (h_k^\top w)^2 \ge t,\quad k=1,\ldots,K,\\
& \|w\|_2^2 \le P,\\
& -\bar w \le w_i \le \bar w,\quad t\ge 0.
\end{aligned}
\label{eq:beamforming_model}
\end{equation}
We generate four instances with four antennas and three users. Channel vectors are sampled from a standard normal distribution, and $\gamma_k=1$ for every user. The power budget is $P=1.05\lVert w_{\mathrm{feas}}\rVert_2^2$, where $w_{\mathrm{feas}}$ is the minimum-norm solution of $h_k^\top w=\sqrt{\gamma_k}$. The componentwise bound is $\bar w=\min\{2,\sqrt{P}\}$.

The reference value is the best solver objective value returned by the general-purpose nonconvex solvers in the application evaluation pipeline. The learned policy and handcrafted references are evaluated by solving the convex surrogates that they generate. Because the beamforming formulation is a maximization problem, the reported relaxation values are upper bounds.

\subsection{Sensor network localization}\label{sec_supp_localization}
The sensor network localization application uses a bounded-noise anchored localization QCQP. Unknown sensor positions \(x_i\in\mathbb{R}^d\) are localized using anchors \(a_k\) and noisy squared-distance observations \(\widehat d^2\). The sets \(E_{ss}\) and \(E_{sa}\) index observed sensor--sensor and sensor--anchor pairs, respectively, and \(\Delta\) is the observation interval half-width. The formulation is
\begin{equation}
\begin{aligned}
\min_{\{x_i\}}\quad & \sum_i \|x_i\|_2^2\\
\mathrm{s.t.}\quad
& \widehat d_{ij}^2-\Delta \le \|x_i-x_j\|_2^2 \le \widehat d_{ij}^2+\Delta,
&& (i,j)\in E_{ss},\\
& \widehat d_{ik}^2-\Delta \le \|x_i-a_k\|_2^2 \le \widehat d_{ik}^2+\Delta,
&& (i,k)\in E_{sa},\\
& \ell_i \le x_i \le u_i.
\end{aligned}
\label{eq:localization_model}
\end{equation}
We generate four two-dimensional instances. Each contains three sensors, three anchors, three sensor--sensor edges, and all nine sensor--anchor edges. Gaussian noise with standard deviation 0.03 perturbs each squared distance. The interval half-width is $\Delta=0.06$. Initial coordinate bounds of $[-3,3]$ are tightened using the anchor observations before relaxation.

The reference values are the best solver objective values returned by the general-purpose nonconvex solvers in the application evaluation pipeline. The learned policy, McCormick endpoint, and SDP endpoint are evaluated as convex lower-bounding relaxations of the minimization problem.

\printbibliography
\end{document}